\documentclass{article}

\usepackage{microtype}
\usepackage{graphicx}
\usepackage{subfigure}
\usepackage{hyperref}
\usepackage[small]{caption}
\usepackage{graphicx}
\usepackage{amsmath}
\usepackage{booktabs}
\usepackage{algorithm}
\usepackage{algorithmic}
\usepackage{bbm}
\usepackage{float}
\usepackage{bm}
\usepackage{tabularx} 
\usepackage{graphicx} 
\usepackage{cite} 
\usepackage{tcolorbox}
\usepackage{amsmath}
\usepackage{amsthm}
\usepackage{multirow}
\usepackage{amssymb}  
\usepackage{tabularx} 
\usepackage{amsmath}  
\usepackage{mathtools}

\newtheorem{theorem}{Theorem}

\newtheorem{corollary}{Corollary}
\newtheorem{lemma}{Lemma}
\newtheorem{assumption}{Assumption}
\newtheorem{proposition}{Proposition}
\newtheorem{definition}{Definition}
\newtheorem{remark}{Remark}

\DeclareMathOperator*{\argmax}{arg\,max}
\usepackage{mathtools}
\DeclarePairedDelimiter{\ceil}{\lceil}{\rceil}
\DeclarePairedDelimiter{\floor}{\lfloor}{\rfloor}
\usepackage{booktabs} 

\usepackage{hyperref}

\usepackage[accepted]{icml2020}

\icmltitlerunning{Cooperative Multiagent Bandits with Heavy Tails}

\begin{document}

\twocolumn[
\icmltitle{Cooperative Multi-Agent Bandits with Heavy Tails}

\begin{icmlauthorlist}
\icmlauthor{Abhimanyu Dubey}{to}
\icmlauthor{Alex Pentland}{to}
\end{icmlauthorlist}

\icmlaffiliation{to}{Media Lab and Institute for Data, Systems and Society,
  Massachusetts Institute of Technology}


\icmlcorrespondingauthor{Abhimanyu Dubey}{dubeya@mit.edu}

\icmlkeywords{multiagent systems, stochastic bandits, heavy tailed, robustness}

\vskip 0.3in
]



\printAffiliationsAndNotice{}  
\begin{abstract}
We study the heavy-tailed stochastic bandit problem in the cooperative multi-agent setting, where a group of agents interact with a common bandit problem, while communicating on a network with delays. Existing algorithms for the stochastic bandit in this setting utilize confidence intervals arising from an averaging-based communication protocol known as~\textit{running consensus}, that does not lend itself to robust estimation for heavy-tailed settings. We propose \textsc{MP-UCB}, a decentralized multi-agent algorithm for the cooperative stochastic bandit that incorporates robust estimation with a message-passing protocol. We prove optimal regret bounds for \textsc{MP-UCB} for several problem settings, and also demonstrate its superiority to existing methods. Furthermore, we establish the first lower bounds for the cooperative bandit problem, in addition to providing efficient algorithms for robust bandit estimation of location.
\end{abstract}
\section{Introduction}
The multi-agent cooperative bandit is an increasingly relevant decision-making problem in the era of large-scale and distributed inference. In this setting, a group of agents, $v \in V$, are each faced with a decision-making problem (e.g., a multi-armed bandit), while simultaneously communicating with other agents over a network. The motivation for studying this problem stems from the rise in decentralized computing systems, and technical interest arises from designing algorithms to leverage cooperation and accelerate collective decision-making~\cite{landgren2016distributed, landgren2018social}.

Consider the case when the agents interact with the classical multi-armed stochastic bandit. For each agent $v \in V$, the problem proceeds in a series of synchronous rounds $t = 1, 2, ..., T$. In each round, each agent selects an action $A_{v, t} \in \mathcal A$, where the space of actions $\mathcal A$ is assumed to be finite and countable $(|\mathcal A|=K)$. It then obtains an i.i.d. reward $x_{v, t}$, with mean $\mu_{A_{v, t}},A_{v, t} \in \mathcal A$. The overall objective of the agents is to minimize the \textit{group regret} $R_\mathcal G(T)$.
\begin{equation}
R_\mathcal G(T) = \sum_{v \in \mathcal G} R_v(T) = \mu_*T|V| - \sum_{v \in V}\sum_{t=1}^T \mu_{A_{v, t}}.
\end{equation}
Here, $\mu_* = \max_{a \in \mathcal A} \mu_a$ denotes the mean of the optimal arm, and $\Delta_a = \mu_*-\mu_a$ is the suboptimality of arm $a$ from the optimal arm $A_* = \argmax_{a \in \mathcal A}\mu_a$. In the cooperative bandit, each agent faces the same problem (i.e., $\mathcal A$ is the same for each agent). In addition to the exploration-exploitation tradeoff succintly captured by the conventional bandit problem, each agent faces the additional complexity of leveraging feedback from its neighbors. This feedback is typically available to the agent after a (variable) \textit{delay}, based on the communication network between the agents. This communication network is conveniently determined by a graph $\mathcal G = (V, E)$, where $V$ denotes the set of agents, and $E$ is the set of edges, where edge $(i, j) \in E$ if agents $i$ and $j$ are neighbors. Messages are sent between trials of the bandit problem, and take $d(v, v')$ trials to be communicated from agent $v$ to $v'$ (where $d(v, v')$ is the distance between the agents $v$ and $v'$ in $\mathcal G$). 

Research on this problem has exclusively been on reward distributions that are sub-Gaussian~\cite{landgren2016distributed, landgren2016distributed2, martinez2018decentralized}. While this is certainly applicable in several domains, increasing evidence suggests that assumptions of sub-Gaussianity may not hold for numerous applications  specific to \textit{distributed} decision-making, in problems such as distributed load estimation of internet traffic~\cite{hernandez2004variable, crovella1998heavy}, multi-agent modeling of supply chain networks~\cite{thadakamaila2004survivability}, modeling information cascades in economic multi-agent models~\cite{de1999information, konovalov2010information} and, among others, numerous problems in distributed modeling for social science~\cite{barabasi2005origin, eom2015tail}. Moreover, in sub-Gaussian environments, if the multi-agent communication is noisy with a (small) non-zero bias, this can lead to model misspecification equivalent to the communicated variable exhibiting heavy tails~\cite{mcculloch2011prediction}. It is therefore prudent to devise methods that are robust to such \textit{heavy-tailed} effects, which is the central theme of this paper. We define heavy-tailed random variables as follows. 

\begin{definition}[Heavy-Tailed Random Variables]
A random variable $X$ is \textit{heavy-tailed} if it does not admit a finite moment generating function, i.e., there is no $u_0>0$ such that,
\begin{equation*}
    \forall |u| \leq u_0, \ M_X(u) \triangleq \mathbb E[\exp(uX)] < \infty.
\end{equation*}
$X$ is $(1+\varepsilon)$-heavy tailed if $\mathbb E[|X|^{t}] = \infty$ for all $t >1+\varepsilon$. 
\end{definition}

Almost all aforementioned research on this problem proposes algorithms that average opinions via a \textit{consensus} protocol~\cite{landgren2016distributed}, and ensure a group regret of $O(\ln T)$ in sub-Gaussian settings. However, these algorithms, as a consquence, employ the empirical mean, which, as we demonstrate in this paper, leads to $O(T^{2/3})$ regret under $2$-heavy tails (i.e., finite variance, see Section~\ref{sec:consensus}). Moreover, the group regret achieved has a $O(\log |V|)$ dependence on $|V|$ (i.e., the number of agents), which we demonstrate to be suboptimal. To our knowledge, no prior work addresses $|V|$-optimality for the stochastic problem\footnote{There is recent research on the \textit{adversarial} cooperative case~\cite{cesa2019delay} with an $\Omega(\sqrt{|V|T})$ lower bound. For the stochastic case, we propose algorithms that achieve lower \textit{problem-dependent} rates, matching our bound of $\Omega(K \ln T)$.}. 

\textbf{Contributions}. Our first contribution is the first (problem-dependent) asymptotic lower bound of $\Omega(K\Delta^{-1/\varepsilon}\ln T)$ on the \textit{group regret} for the cooperative multi-armed stochastic bandit. This result holds for any connected graph $\mathcal G$ and arbitrary communication protocol (following mild conditions). Note that this implies that agents on average can achieve a maximum of $O(\tfrac{1}{|V|})$ reduction in regret when cooperating, compared to the earlier benchmark of $O(\tfrac{\ln |V|}{|V|})$ \cite{landgren2016distributed, martinez2018decentralized}. These results generalize the lower bound for the stochastic bandit with multiple pulls, obtained by \citet{anantharam1987asymptotically}, to the case when rewards are obtained after delays. In the heavy-tailed case, this lower bound matches the $\Omega(\Delta^{-1/\varepsilon}\ln T)$ rate obtained by \citet{bubeck2013bandits}, and in the sub-Gaussian case, matches the $\Omega(\Delta^{-1}\ln T)$ problem-dependent rates~\cite{agrawal2012analysis}.

Next, we present an algorithm \textsc{MP-UCB} for the cooperative multi-agent stochastic bandit under heavy-tailed densities. The key concept utilized in the development of the algorithm is to provide an alternate technique to control the variance of the arm estimators across the network $\mathcal G$ when the \textit{consensus} protocol provides suboptimal guarantees. This is done by utilizing an alternate communication protocol titled \textsc{Local}~\cite{linial1992locality} to share information between agents, and then incorporating robust mean estimators to achieve optimal rates. In this process, we also outline a subroutine to efficiently compute the univariate robust mean for the bandit problem (i.e., when confidence $\delta$ changes with time).

We demonstrate that \textsc{MP-UCB} achieves $O\left(\bar\chi(\mathcal G_\gamma)K\ln T\right)$ group regret when run in a completely decentralized manner (i.e., agents select actions independently), and $O\left(\alpha(\mathcal G_\gamma)K\ln T\right)$ when run in a centralized manner (i.e., some agents mimic others), similar to the regret bounds obtained by~\citet{cesa2019delay} for the adversarial case. Here, $\gamma \leq \text{diameter}(\mathcal G)$ is a parameter controlling the density of communication, and $\mathcal G_\gamma$ is the $\gamma$ graph power of $\mathcal G$. $\bar\chi$ denotes the clique covering number, and $\alpha \leq \bar\chi$ denotes the independence number of a graph. These results are optimal, in the sense that when our algorithm is run with $\gamma = \text{diameter}(\mathcal G)$, both variants obtain a group regret of $O(K \ln T)$, matching the lower bound. This $O(\ln |V|)$ improvement is achieved since the \textsc{Local} protocol allows us to partition the power graph $\mathcal G_\gamma$ in a manner that induces a constant regret overhead from the communication delay, in contrast to the \textit{consensus} protocol that diffuses information slowly for sparse $\mathcal G$. Furthermore, when we allow $O(K)$-sized messages per round, we demonstrate that \textsc{MP-UCB} obtains $O(\alpha(\mathcal G_\gamma)K\ln T)$ regret without knowledge of $\mathcal G$. 

We evaluate our algorithms on a benchmark of real-world and random graphs. While we consider heavy-tailed densities, it can be easily seen that \textsc{MP-UCB} can be applied to sub-Gaussian densities with optimal rates as well.

\section{Related Work}
\textbf{Cooperative Decision-Making}. Cooperative decision-making for the stochastic multi-armed bandit has recently seen a lot of research interest. Decentralized cooperative estimation has been explored for sub-Gaussian stochastic bandits using a \textit{running consensus} protocol in~\cite{landgren2016distributed, landgren2016distributed2, martinez2018decentralized} and for adversarial bandits~\cite{cesa2019delay, bar2019individual} using a message-passing protocol. Localized decision-making for sub-Gaussian rewards has also been explored in the work of~\cite{landgren2018social}, and a fully-centralized algorithm in~\cite{shahrampour2017multi}, where all agents select the same action via voting.The stochastic bandit with multiple pulls~\cite{xia2016budgeted, anantharam1987asymptotically} is equivalent to the cooperative multi-armed bandit on a complete $\mathcal G$ with a centralized actor (since there are no delays and all agents have the same information $\forall t \in [T]$).

Contrasted to cooperative settings, there is extensive research in competitive settings, where multiple agents compete for arms~\cite{bistritz2018distributed, bubeck2019non, liu2010distributed, liu2010distributed2, liu2010decentralized}. For strategic experimentation,~\citet{branzei2019multiplayer} provide an interesting comparison of exploration in cooperative and competitive agents.

A closely-related problem setting is the single-agent~\textit{social network} bandit, where a user is picked at random every trial, and the algorithm must infer its contextual mean reward~\cite{cesa2013gang, li2016collaborative, gentile2014online, gentile2017context}, while assuming an underlying ~\textit{clustering} over the users. This problem setting, while relevant, crucially differs from the one considered herein, since (a) this is a single-agent setting (only one action is taken every round), and (b) there are no delays in the rewards obtained. While a multi-agent variant has been considered~\cite{korda2016distributed}, this work also assumes no delays in communication. 

\textbf{Heavy-Tailed Bandits}. \citet{bubeck2013bandits} first discuss the problem of stochastic bandits with heavy-tailed rewards, and propose the~\textsc{Robust-UCB} algorithm that uses robust mean estimators to obtain logarithmic regret. \citet{vakili2013deterministic} introduce DSEE, an algorithm that sequences phases of exploration and exploitation to obtain sublinear regret. Thompson Sampling~\cite{thompson1933likelihood} has been analysed for exponential family bandits (that include Pareto and Weibull heavy-tailed distributions) in the work of~\citet{korda2013thompson}, however, these distributions have ``lighter'' tails owing to the existence of higher order moments. \citet{dubey2019thompson} provide an algorithm for Thompson Sampling for $\alpha$-stable densities~\cite{borak2005stable}, at family of heavy-tailed densities typically with infinite variance.~\citet{yupure} provide a purely exploratory algorithm for best-arm identification for $\varepsilon$-heavy tailed rewards. For the linear bandit,~\cite{shao2018almost, medina2016no} provide nearly-optimal algorithms under heavy tails. To our knowledge, this paper is the first to study robust bandit learning in the context of decentralized multi-agent estimation.
\section{Preliminaries}
\textbf{Finite-Armed Stochastic Bandit}. We consider the family of bandit problems $\mathcal E$ for a finite, countable set of actions $\mathcal A$, such that $|\mathcal A| = K$. $\mathcal E$ is considered to be \textit{unstructured}, i.e. the rewards from each arm are independent of the others.
\begin{definition}[Unstructured Bandit Problem]
An environment class of bandit problems $\mathcal E$ is unstructured if its action space $\mathcal A$ is finite, and there exists a set of distributions $\mathcal M_a$ for each $a \in \mathcal A$ such that
\begin{equation*}
    \mathcal E = \{\bm\nu = (P_a : a \in \mathcal A) : P_a \in \mathcal M_a \forall a \in \mathcal A\}.
\end{equation*}
\end{definition}
Agents face a common stochastic bandit with $K$ arms. Rewards from arm $k \in [K]$ are drawn from an $\varepsilon$-heavy tailed distribution $\nu_k$ with mean $\mu_k$, and known bounds on the $(1+\varepsilon)$ moments $\mathbb E[|X-\mu_k|^{1+\varepsilon}] \leq \rho$ and $\mathbb E[|X|^{1+\varepsilon}] \leq u$. The optimal arm is given by $k^* = \arg\max_{k \in [K]}\mu_k$.

\textbf{Cooperative Problem Setting}. We consider $M$ agents communicating via a connected, undirected graph $\mathcal G = (V, E)$. Communication is bidirectional, and any message sent from agent $v$ is obtained by agent $v'$ after $d(v, v') - 1$ rounds of the bandit problem, where $d(v, v')$ denotes the length of the shortest path between the agents. Let $\bm L$ and $\bm A$ denote the graph Laplacian and adjacency matrix of $\mathcal G$, and ${\bm P} = {\bm I}_M - \kappa\cdot d_{\max}^{-1}{\bm L}$ is a row stochastic matrix, where ${\bm I}_M$ is the identity matrix of order $M$, $\kappa > 0$ is a constant and $d_{\max} = \max_{m \in \mathcal G} \text{degree}(m)$. We assume that the eigenvalues $\lambda_i$ of $\bm P$ are ordered such as $\lambda_1 =1 \geq ... \geq \lambda_M > -1$.  Let $A_{mt}$ denote the action taken by agent $m$ at time $t$, and $X_{mt}$ denote the corresponding reward. $n_k(T)$ denotes the total number of times any arm $k$ is pulled across all agents, and $n^m_k(T)$ denotes the times agent $m$ has pulled the arm. Let the power graph of order $\gamma$ of $\mathcal G$ be given by $\mathcal G_{\gamma}$, i.e., $\mathcal G_\gamma$ contains an edge $(i, j)$ if there exists a path of length at most $\gamma$ in $\mathcal G$ between agents $i$ and $j$. For any agent $v \in V$, let the neighborhood of $m$ in $\mathcal G_\gamma$ be given by $\mathcal N_\gamma(v)$. The policy of agent $v \in V$ is given by $(\pi_{v, t})_{t \in [T]}$, and the collective policy is given by $\Pi = (\pi_{v, t})_{v \in V, t \in [T]}$. 
\begin{definition}[Consistent Bandit Policy]
Let $\Pi$ be any bandit policy, potentially running over multiple agents. $\Pi$ is \textbf{consistent} if, for any suboptimal arm $k \in [K], k \neq k^*$, horizon $T > 0$, one has $\mathbb E[N_k(T)] = o(T^a)$ for any $a > 0$.
\end{definition}

\textbf{Univariate Robust Estimation}. Optimal algorithms have been proposed for robust estimation of location in the univariate setting with polynomial running time. The simplest of these is the trimmed mean, that rejects outlying samples based on an upper bound on the moments. Its runtime for $N$ samples obtained sequentially (with changing confidence $\delta$) is $ O(N^2)$, which we improve to $O(N \ln N)$ (Algorithm~\ref{alg:online_mean}).
\begin{definition}[Trimmed Mean]
Consider $n$ copies $X_1, ..., X_n$ of a heavy-tailed random variable $X$ such that $\mathbb E[X] = \mu, \mathbb E[X^{1+\varepsilon}]\leq u$ for some $\varepsilon\in(0, 1]$. The online trimmed mean, for some $\delta \in (0, 1)$ is defined as
\begin{equation*}
    \hat{\mu}_O = \frac{1}{n}\sum_{i=1}^n X_i \bm{1}\left\{|X_i| \leq \left(\frac{ui}{\log \delta^{-1}}\right)^{\frac{1}{1+\varepsilon}}\right\}.
\end{equation*}
\label{def:trimmed_mean}
\end{definition}
Several alternative robust mean estimators exist, such as the median-of-means or Catoni's estimator~\cite{catoni2012challenging}. Under stricter tail assumptions, they provide better estimates, however, for simplicity, we continue with the trimmed mean. In the analysis, we assume that a mean estimator exists that achieves the following optimal rate (up to constants).
\begin{assumption}[Rate Assumption]
Let $X_1, ..., X_n$ be $n$ samples of an $\varepsilon$-heavy tailed random variable, where $\varepsilon \in (0, 1]$, and $\mathbb E[X] = \mu$. For positive constants $c, \rho$ suppose that there exists a robust estimator $\hat\mu(\delta, n)$ such that, with probability at least $1-\delta$,
\begin{equation*}
    |\hat\mu(\delta, n) - \mu| \leq 2 \rho^{\frac{1}{1+\varepsilon}}\left(\frac{c\log(\delta^{-1})}{n}\right)^{\frac{\varepsilon}{1+\varepsilon}}.
\end{equation*}
\label{ass:robust_mean_estimator}
\end{assumption}
\citet{catoni2012challenging} provides this as the optimal achievable rate under heavy tails, and \citet{bubeck2013bandits} demonstrate that the trimmed mean achieves this rate (see appendix).
\section{Lower Bounds}
We now present lower bounds on cooperative decision-making. All full proofs are presented in the appendix for brevity. We consider $|V| = M$ agents communicating over graph $\mathcal G$, with $\text{diameter}(\mathcal G) = \gamma_* \ll M$. We first make some (mild) assumptions on the communication protocol.
\begin{assumption}[Communication Protocol]
The communication protocol considered follows:
\begin{enumerate}
    \item Any agent $m$ is capable of sending a message $\bm q_{m}(t)$ to any other agent $m' \in [M]$, which is earliest received at time $t+\min(0, d(m, m') - 1)$.
    \item $\bm q_{m}(t)$ is a function of the action-reward pairs of agent $m$, i.e. $\bm q_{m}(t) = {\bm F}_t(A_{m, 1}, X_{m, 1}, ..., A_{m, t}, X_{m, t})$ for any deterministic, bijective and differentiable set of functions ${\bm F}_t = (f_{i, t})_{i \in [L]}, f_{i,t}: \mathbb R^{2t} \rightarrow \mathbb R$.
    \item  ${\bm F}_t$ satisfies $|\text{det} \left({\bm J}_t\right)| = \Lambda(m, t).$ Here, ${\bm J}_t(\cdot)$ is the Jacobian of ${\bm F}_t$, and $\Lambda$ is only a function of $m$ and $t$.
\end{enumerate}
\label{ass:comm_protocol}
\end{assumption}
This assumption ensures that (a) information can flow between any two agents, and (b) that the messages are not stochastic and are independent the bandit problem. We can then derive a lower bound on the group regret.

\begin{theorem}[Lower Bound]
For any consistent cooperative multi-agent policy $\Pi = (\Pi_t)_{t \in [T]}$ on $M$ agents that satisfies Assumption~\ref{ass:comm_protocol} the following is true.
\begin{equation*}
  \liminf_{T\rightarrow\infty}\frac{R_\mathcal G(T)}{\ln T} \geq \sum_{k:\Delta_k > 0} \frac{\Delta_k}{\mathbb D^{\mathsf {inf}}_k}.
\end{equation*}
Here, $\mathbb D^{\mathsf{inf}}_k = \inf_{\nu' \in \mathcal M_k}\left\{\mathbb D_{\mathsf {KL}}(\nu, \nu') : \mu(\nu') > \mu^*\right\}$, and $\mathbb D_{\mathsf{KL}}(\cdot, \cdot)$ denotes the Kullback-Leibler divergence.
\label{thm:lower_bound_pulls}
\end{theorem}
\begin{remark} 
Theorem~\ref{thm:lower_bound_pulls} does not guarantee an overhead from delayed communication, since it includes protocols that allow information to flow completely through the (connected) network $\mathcal G$, albeit at a delay (which is independent of $T$). Making stronger assumptions about the connectivity of $\mathcal G$ and communication protocol can lead to stronger bounds.
\end{remark}
\begin{remark}
This result generalizes that obtained by~\citet{anantharam1987asymptotically} for a centralized agent with multiple pulls to the case where rewards are obtained after finite delays. This can be understood by considering a complete $\mathcal G$, which is equivalent to having a centralised agent (since there is no difference in information between agents). The comparison with a single agent pulling $MT$ arms\cite{martinez2018decentralized}, is therefore an incorrect benchmark.
\end{remark}
For the specific case of $(1+\varepsilon)$-heavy tailed rewards, the single-agent lower bound provided by~\cite{bubeck2013bandits} can be easily extended to the cooperative multi-agent case.
\begin{corollary}[Lower Bound on Heavy-Tailed Cooperative Regret]
For any $\Delta \in (0, 1/4)$, there exist $K \geq 2$ distributions $\nu_1, ..., \nu_K$ satisfying $\mathbb E_{X \sim \nu_k}[|X|^{1+\varepsilon}] \leq u$, and $\mathbb E_{X \sim \nu_*}[X] - \mathbb E_{X \sim \nu_k}[X] = \Delta \forall k \in K$, such that any consistent decentralized policy $\Pi_t = (\pi_{m, t})_{m \in [M], t \in [T]}$ that satisfies Assumption~\ref{ass:comm_protocol} obtains group regret of $\Omega(K\Delta^{-1/\varepsilon}\ln T)$.
\end{corollary}
The $O(\Delta^{-1/\varepsilon})$ dependency is unavoidable, as shown in~\cite{bubeck2013bandits}, and it can be matched using robust estimators. The formulation of robust estimators makes averaging-based communication protocols infeasible, such as the \textit{running consensus}, as shown in the following section.
\section{The Limits of Running Consensus}
\label{sec:consensus}
Under the consensus protocol, agents maintain an estimate of values of interest, which they average with their neighbors every round. The protocol stores $2K$ opinion vectors $\hat{\bm s}_k(t) = (\hat s^v_k(t))_{v \in V}, k \in [K]$ and $\hat{\bm n}_k(t) = (\hat n^v_k(t))_{v \in V}, k \in [K]$, that are updated as follows.
\begin{align}
    \hat{\bm s}_k(t) &= {\bm P} \left(\hat{\bm s}_k(t-1) + {\bm r}_k(t) \circ  {\bm \zeta}_k(t)\right). \\
    \hat{\bm n}_k(t) &= {\bm P} \left(\hat{\bm n}_k(t-1) + {\bm \zeta}_k(t)\right).
\end{align}
Here $\hat{\bm s}_k(t)$ is a vector of reward sums for arm $k$ for each agent, ${\bm \zeta}_k(t)$ is a vector of indicators of whether the agents pulled arm $k$ at time $t$, and ${\bm r}_k(t)$ is the vector of rewards obtained by the agents from arm $k$. Using this, any agent $v \in V$ computes the empirical mean of each arm $k$.
\begin{equation}
    \hat\mu_k ^{(v)}(t) = \hat{s}^{(v)}_k(t)/\hat{n}^{(v)}_k(t).
\end{equation}

When $\varepsilon = 1$, i.e. the reward distributions have finite variance, we can  design a UCB algorithm \textsc{Consensus-UCB}~\cite{landgren2016distributed}, where each agent chooses the arm that maximizes the following UCB. 
\begin{equation}
    A_{v, t} = \arg\max_{k \in [K]} \left\{\hat\mu_k ^{(v)}(t) +\sqrt{\frac{6\rho t^{2/3}}{|V|}\left(\frac{\hat{n}^v_k(t) + \epsilon^k}{\hat{n}^v_k(t)^2}\right)}\right\}.
\end{equation}

\begin{theorem}
The \textsc{Consensus-UCB} algorithm obtains a group regret of $O\left((1+h(\mathcal G))KT^{\frac{2}{3}}\right)$ after $T$ trials, where $h(\mathcal G)$ is, for constants $a_{pj}$ that only depend on $\mathcal G$,
\begin{equation*}
  h(\mathcal G) = \sum_{p=1}^{|V|} \sum_{j=2}^{|V|} \frac{|\lambda_p\lambda_j|}{1-|\lambda_p\lambda_j|}a_{pj}.
\end{equation*}
\label{thm:coop_ucb_basic}
\end{theorem}
A full description of this algorithm is included in the appendix. Since we are utilizing the empirical mean for \textsc{Consensus-UCB}, the UCB utilized cannot be made tighter, suggesting that the algorithm is suboptimal in $T$. More importantly, it can be expected that any algorithm that utilizes the empirical mean (running consensus converges to the empirical mean as $T \rightarrow \infty$~\cite{aysal2010convergence}) is suboptimal in $T$ owing to the suboptimality of the empirical mean itself (see appendix).
\section{Message-Passing Cooperative UCB}
In the message-passing protocol, agents $v \in V$ communicate via messages ${\bm q}_v(t) = \left\langle v, t, A_{v, t}, X_{v, t} \right\rangle$. This message is first sent to its neighbors in $\mathcal G$, and it is subsequently forwarded by any agent that receives it until time $t+\gamma$, after which it is discarded. $0 \leq \gamma \leq \text{diameter}(\mathcal G)$ is therefore the communication \textit{density}, where lower values of $\gamma$ imply less communication in the network.

Let $Q_v(t)$ denote the set of incoming messages received by agent $v$ at instant $t$. During any trial, the agent first pulls an arm, and creates the message $\bm q_v(t)$. It then processes all messages in $Q_v(t)$, and updates its beliefs as per any bandit algorithm. Finally, it discards all messages older than $t-\gamma$ and forwards all remaining messages in $Q_v(t) \cup \{\bm q_v(t)\}$ to all its neighbors in $\mathcal G$. This protocol has been used in distributed optimization~\cite{moallemi2007message}, non-stochastic bandit settings~\cite{cesa2019delay, bar2019individual} and asynchronous online learning~\cite{suomela2013survey}. This protocol satisfies Assumption~\ref{ass:comm_protocol} with $\gamma=\text{diameter}(\mathcal G)$.
\begin{algorithm}[t]
\caption{\textsc{Decentralized MP-UCB}}
\small
\label{alg:mp_ucb_unknown_G}
\begin{algorithmic}[1] 
\STATE \textbf{Input}: Arms $k \in [K]$, parameters $\varepsilon, c, \rho$, estimator $\hat\mu(n, \delta)$
\STATE $S^v_k \leftarrow \phi\ \forall k \in [K]$, $Q_v(t) \leftarrow \phi$, $\forall v \in V$.
\FOR{each iteration $t \in [T]$}
\FOR{each agent $v \in V$}
\IF{$t \leq K$}
\STATE $A_{m, t} \leftarrow t$.
\ELSE
\FOR{Arm $k \in [K]$}
\STATE $\hat\mu^{(v)}_k \leftarrow \hat\mu(S^v_k, 1/t^2)$.
\STATE $\text{UCB}_k^{(v)}(t) \leftarrow \rho^{\frac{1}{1+\varepsilon}}\left(\frac{2c\ln t}{|S^v_k|}\right)^{\frac{\varepsilon}{1+\varepsilon}}$.
\ENDFOR
\STATE $A_{v, t} \leftarrow \arg\max_{k \in [K]} \left\{\hat\mu_k ^{(v)}(t) + \text{UCB}_k^{(v)}(t)\right\}$.
\ENDIF
\STATE $X_{v, t} \leftarrow\textsc{Pull}(A_{v, t})$.
\STATE $S_{A_{v,t}}^v \leftarrow S_{A_{v,t}}^v \cup \{X_{v, t}\}$
\STATE $Q_v(t) \leftarrow Q_v(t) \cup \{\left\langle v, t, A_{v, t}, X_{v, t}\right\rangle\}$.
\FOR{each neighbor $v'$ in $\mathcal N_1(v)$}
\STATE \textsc{SendMessages}$(v, v', Q_v(t))$.
\ENDFOR
\ENDFOR
\FOR{each agent $v \in V$}
\STATE $Q_v(t+1) \leftarrow \phi$.
\FOR{each neighbor $v'$ in $\mathcal N_1(v)$}
\STATE $Q' \leftarrow $\textsc{ReceiveMessages}$(v',v)$
\STATE $Q_v(t+1) \leftarrow Q_v(t+1) \cup Q'$.
\ENDFOR
\FOR{$\langle v',t',a',x'\rangle \ \in Q_v(t+1)$}
\IF{$v' \in \textsc{Clique}(v, \mathcal G_\gamma)$}
\STATE $S_{a'}^v \leftarrow S_{a'}^v \cup \{x'\}$.
\ENDIF
\ENDFOR
\ENDFOR
\ENDFOR
\end{algorithmic}
\end{algorithm}
\subsection{Decentralized Algorithm}
In the decentralized setting, each agent acts independently, i.e., there is no centralized controller that dictates actions. In this setting, each agent $v$ maintains a set $S^v_k(t)$ of rewards obtained from arm $k$, which it updates at each trial from its own pulls and incoming messages. Then it computes the robust mean of $S^v_k(t)$ via the estimator $\hat\mu(|S^m_k(t)|, \delta)$. Using Assumption~\ref{ass:robust_mean_estimator}, it then estimate a UCB for each arm mean, and selects the arm with the largest UCB (Algorithm~\ref{alg:mp_ucb_unknown_G}).
\begin{theorem}
The group regret for Algorithm~\ref{alg:mp_ucb_unknown_G} when run with parameter $\gamma$ and mean estimator $\hat\mu(n, \delta)$ that satisfies Assumption~\ref{ass:robust_mean_estimator} with constants $c$ and $\rho$ satisfies:
\begin{multline*}
    R_\mathcal G(T) \leq C\bar\chi\left(\mathcal G_\gamma\right)\left(\sum_{k : \Delta_k > 0} (2\Delta_k)^{-1/\varepsilon}\right)\ln T +\\\left(3M + \gamma\bar{\chi}\left(\mathcal G_\gamma\right)\left(M-1\right)\right)\left(\sum_{k : \Delta_k > 0}\Delta_k\right).
\end{multline*}
Here, $C > 0$ is a constant independent of $T, K, M$, and $\bar\chi(\cdot)$ refers to the clique number.
\end{theorem}

\textit{Proof (sketch)}. We first bound the regret in each clique $\mathcal C$ within the clique covering $\bm C_\gamma$ of $\mathcal G_\gamma$. This is done by noticing that the upper confidence bound for any arm at a selected $t$ deviates by a constant amount between agents based on the number of times each agent has pulled an arm. By bounding this deviation, we obtain a relationship between the confidence bound of each arm for each agent within the clique $\mathcal C$. Next, we bound the probability of pulling a suboptimal arm within the clique $\mathcal C$ using the previous result. Summing over the clique cover $\bm C_\gamma$ delivers the final form of the result. The complete proof is included in the appendix for brevity.
\begin{remark}
Communication density determines the group regret dependence or ``cooperation'' in Algorithm~\ref{alg:mp_ucb_unknown_G}. When $\gamma = \text{diameter}(\mathcal G), \bar\chi(\mathcal G_\gamma)=1$, and we incur optimal group regret $O(K \Delta^{-1/\varepsilon}\ln T)$, and also satisfies both assumptions of Assumption~\ref{ass:comm_protocol}. However, when $\gamma = 0$, i.e. agents do not communicate, regret is $O(|V|K \Delta^{-1/\varepsilon}\ln T)$.
\end{remark}
Each agent in Algorithm~\ref{alg:mp_ucb_unknown_G} utilizes observations only from its own clique in $\mathcal G_\gamma$ to make decisions, effectively paritioning $\mathcal G$. When $\mathcal G$ is sparse (e.g., small-world networks~\cite{barabasi2005origin}), the clique number of the graph $\mathcal G_\gamma$ can be large. In this case, a centralized variant can provide lower regret.

\subsection{Centralized Algorithm}
In the centralized setting, we present a version of the ``follow-the-leader'' strategy. Here, the agents are partitioned into ``leaders'' and ``followers''. The leader agents follow the same procedure identically to Algorithm~\ref{alg:mp_ucb_unknown_G}, and the follower agents simply copy the most recent action they have observed of their associated leader. We now describe how the graph $\mathcal G$ is partitioned into leaders and followers.
\begin{definition}[Maximal Weighted Independent Set]
An indepedent set of a graph $\mathcal G = (V, E)$ is a set of vertices $V' \subseteq V$ such that no two vertices in $V'$ are connected. A maximal independent set $V^*$ is the largest independent set in $\mathcal G$, and the independence number $\alpha(\mathcal G) = |V^*|$. For a vertex-weighted graph, a maximal weighted independent set $V_w' \subseteq V$ is the maximal independent set such that the sum of weights for all vertices in $V_w'$ is the largest possible.
\end{definition}

We select the leaders as the members of a maximal independent set $V' \subseteq V$ of $\mathcal G_\gamma$. For each follower agent $v \in V \setminus V'$ we assign a leader $l(v)$ to it such that $(a)$ there is an edge between $v$ and $l(v)$ in $\mathcal G_\gamma$, and $(b)$ $l(v)$ has maximum degree in $V' \cap \mathcal N_\gamma(v)$, i.e. $l(v) \in V'$ such that $l(v) = \arg\max_{v' \in V' \cap \mathcal N_1(v)} \text{degree}(v)$. It is trivial to demonstrate that each agent will either be a leader node, or be connected to a leader (see appendix). Algorithm~\ref{alg:ftl_mp_ucb} describes this algorithm particularly from its differences with the decentralized version.
\begin{algorithm}[t!]
\caption{\textsc{Centralized MP-UCB}}
\label{alg:ftl_mp_ucb}
\small
\begin{algorithmic}[1] 
\STATE \textbf{Input}: Same as Algorithm~\ref{alg:mp_ucb_unknown_G}.
\STATE Set $S^v_k \leftarrow \phi\ \forall k \in [K]$, $Q_v(t) \leftarrow \phi$, $A^*_v \leftarrow 1$, for all $v \in V$.
\FOR{each iteration $t \in [T]$}
\FOR{each agent $v \in V$}
\IF{$t \leq K$}
\STATE $A_{v, t} \leftarrow t$.
\ELSIF{$v \in V'$ \OR $t \leq d(v, l(v))$}
\STATE Run lines 8-12 of Algorithm~\ref{alg:mp_ucb_unknown_G}.
\ELSE
\STATE $A_{v, t} \leftarrow A^*_v$.
\ENDIF
\STATE Run lines 14-19 of Algorithm~\ref{alg:mp_ucb_unknown_G}.
\ENDFOR
\FOR{each agent $v \in V$}
\STATE Run lines 22-26 of Algorithm~\ref{alg:mp_ucb_unknown_G}.
\FOR{$\langle v',t',a',x'\rangle \ \in Q_v(t+1)$}
\STATE $S_{a'}^v \leftarrow S_{a'}^v \cup \{x'\}$.
\ENDFOR
\STATE $A^*_v = $\textsc{ChooseLastAction}($\cup_k S_k^v(t+1)$). 
\ENDFOR
\ENDFOR
\end{algorithmic}
\end{algorithm}
\begin{theorem}
Algorithm~\ref{alg:ftl_mp_ucb} run with parameters $\gamma, c, \rho$ obtains the following group regret  (where $\alpha(\cdot)$ denotes the independence number).
\begin{equation*}
    R_\mathcal G(T) = O\left(\alpha(\mathcal G_\gamma)\left(\sum_{k:\Delta_k>0}\Delta_k^{-1/\varepsilon}\right)\ln T\right).
\end{equation*}
\end{theorem}
\textit{Proof (sketch)}. The key idea is to partition $\mathcal G_\gamma$ into non-overlapping sets given by $V'$ and to note that $R_v \leq R_{l(v)} + \gamma$ for any $v \in V \setminus V'$. Then, we can bound the number of times any element in $V'$ selects an arm until time $t$ as a function of its neighborhood in $\mathcal G_\gamma$. Using this bound, we can then create an UCB to bound the probability of pulling a suboptimal arm for any agent $v \in V'$, and collectively bound the group regret of the entire neighborhood. Summing over $v \in V'$ delivers the final result, since $V'$ forms a vertex cover in $\mathcal G_\gamma$. The complete proof is available in the appendix.

Since $\alpha(\mathcal G) \leq \bar\chi({\mathcal G})$ for any graph $\mathcal G$, the centralized version of the MP-UCB algorithm obtains regret strictly no worse compared to the decentralized version. We are aware that the set of leader nodes must form a maximal independent set in $\mathcal G_\gamma$, however, for large graphs there may be multiple maximal independent sets present, and selecting a suboptimal independent set can increase group regret. Agents present more ``centrally'' may be a better choice as leaders, compared to ``peripheral'' agents. Our choice of independent set is motivated by the following result.
\begin{corollary}
For agent $v \in \mathcal G$, let $v^*$ denote its corresponding leader agent ($v = v^*$ for leaders), and $F(v^*)$ denote the corresponding set of follower agents for $v^*$ (including $v^*$). The following holds for the regret $R_v(T)$.
\begin{align*}
    R_v(T) = O\left(\frac{K\ln T}{|F(v^*)|\Delta^{1/\varepsilon}_{\min}}\right).
\end{align*}
\end{corollary}
We see that, intuitively, even for agents that themselves are not well-connected, as long as they are connected to a well-connected leader (with large $|F(v^*)|$), the individual regret will be low. By this result, we select the weight assigned to any agent $v$ as its degree in $\mathcal G_\gamma$, since, asymptotically performance depends on $(|F(v^*)|^{-1})$. A few additional remarks can be made, inspired by~\citet{bar2019individual}.
\begin{remark}
The average regret from Algorithm~\ref{alg:ftl_mp_ucb} is $O((\alpha(\mathcal G_\gamma)/|V|)K\ln T)$, i.e. optimal when $\gamma = \text{diam}(\mathcal G)$. When $\gamma = \sqrt{K}$, Algorithm~\ref{alg:ftl_mp_ucb} can obtain a per-agent regret of $O(\Delta_*^{-1/\varepsilon}\sqrt{K}\ln T)$. This can be shown following the procedure in~\citet{bar2019individual}, by noticing that when $\mathcal G$ is connected, $\alpha(\mathcal G_\gamma) \leq \ceil*{2|V|/(\gamma + 2)}$. Also note that we need only $\sqrt{K}$ leaders at most to obtain this regret. When $\gamma = \text{diam}(\mathcal G)$, then, only 1 arbitrarily chosen leader can deliver optimal regret, regardless of its position in $\mathcal G$.
\end{remark}
\subsection{Additional Optimizations}
\textbf{$O(K)$ Per-round Communication}. We now demonstrate that communicating additional information beyond just action-reward pairs can significantly improve performance, and obtain optimal regret \textit{without knowledge of $\mathcal G$}. In this case, the message $\bm q_v(t)$ is augmented as follows.
\begin{equation}
\resizebox{0.7\linewidth}{!} 
{$\bm q_v(t) = \left\langle v, t, A_{v, t}, X_{v, t}, \hat{\bm \mu}^{(v)}(t), \bm N^v(t) \right\rangle.$}
\end{equation}
Where $\hat{\bm \mu}^{(v)}(t) = (\hat\mu^{(v)}_k(t))_{k \in [K]}$ are the robust mean estimates used by agent $m$ to make decisions at time $t$, and $\bm N^v(t) = (|S^v_k(t)|)_{k \in [K]} $ is the vector containing the number of reward samples possessed by agent $v$ until time $t$. Each agent $v$ also maintains a set $W$ of the most recent $(\hat{\bm{\mu}}^{(v')}(t), \bm N^{v'}(t))$ for each $v' \in \mathcal N_\gamma(v) \cup \{v\}$, which they update with each message received from agent $v'$. At any instant, the agent chooses, for each arm $k$, the corresponding $\hat\mu_k^{v^*}(t)$ and $N^{v^*}(t)$ in $W$ with the largest $N^{v^*}(t)$ (and ``tightest'' UCB) to construct its upper confidence bound. The full algorithm (KMP-UCB) is described in the appendix.
\begin{figure*}[th]
\centering
  \includegraphics[width=\linewidth]{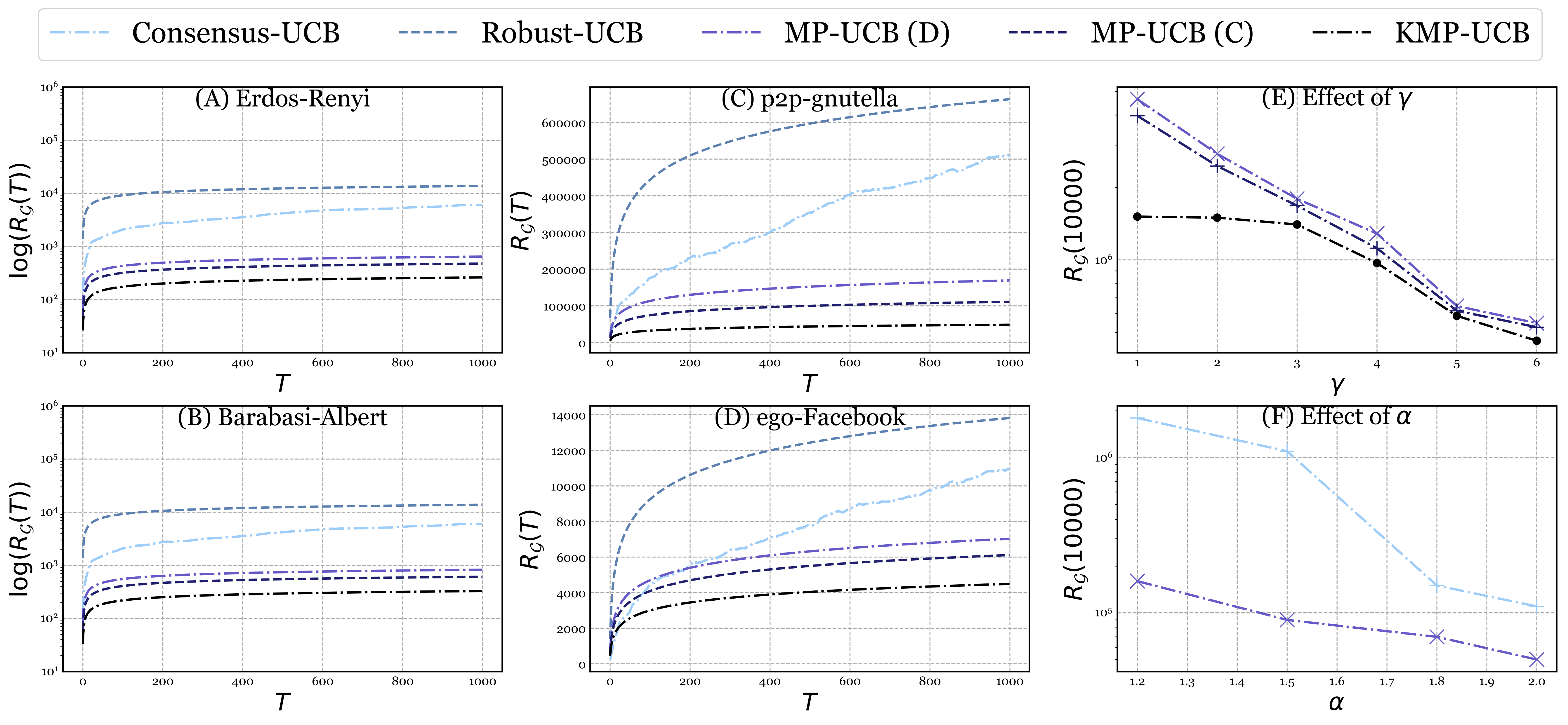}
  \caption{Experimental benchmarks, where each experiment is averaged over 100 trials. Figures (A) and (B) compare performance on samples of random graphs; (C) and (D) compare performance on two classes of real-world networks, and (E) and (F) are ablations.}
  \label{fig:main}
\end{figure*}
\begin{theorem}
\textsc{KMP-UCB} obtains group regret $R_{\mathcal G}(T)$ of  $O(\alpha(\mathcal G_\gamma)K\Delta^{-1/\varepsilon}\ln T)$ over any connected graph $\mathcal G$.
\end{theorem}
\textit{Proof (sketch)}. We first note that there will be an independent set of agents $V' \in V$ that has, at any given $t$, the largest set of observations within their neighborhoods. Since at trial $t+\text{diameter}(\mathcal G_\gamma)$, any other agent will either use the confidence estimates of $v \in V'$, or will have better estimates (from more samples). This provides us a technique to lower bound the number of times the entire group of agents $V$ will pull an arm at any time $t$ in terms of the pulls of $v \in V$, and then construct a UCB for each arm from it. We then proceed by the standard UCB technique for a single-agent, and use concentration of the robust mean to derive regret for each $v \in V'$. Finally, summing over $v$ gives us the desired result.

Contrasted to the regret bound of $O(\sqrt{|V|\alpha(\mathcal G)TK\ln K})$ obtained by~\citet{cesa2019delay} for the nonstochastic case (where communication is also $O(K)$ per agent), our algorithm obtains lower group regret in the stochastic case. Additionally, this implies a $O(\ln |V|)$ improvement over the previous bound in the stochastic case~\cite{martinez2018decentralized}.

\textbf{Online Estimation of Trimmed Mean}. The trimmed mean estimator requires selecting a sample $X_i$ at time $t$ only if $|X_i| \leq (2ui\ln t)^{1/(1+\varepsilon)}$ (Definition \ref{def:trimmed_mean}). This implies that the $i^{th}$ reward sample an agent has will be selected at the smallest time $t$ such that $(|X_i|^{1+\varepsilon}/(i)) \leq 2u\ln t$. When $T$ is knowns, we can utilize a binary search tree to make an update to the robust mean $O(\ln t)$ instead of $O(t)$ at time $t$. We outline this procedure in Algorithm~\ref{alg:online_mean}. 
\begin{algorithm}[t]
\caption{\textsc{Online Trimmed Mean Estimator}}
\small
\label{alg:online_mean}
\begin{algorithmic}[1] 
\STATE \textbf{Input}: $u, T$.
\STATE Create dictionary $D$ of size $T$, where $D(t) = \phi \ \forall t \in [T]$.
\STATE Create BST $B$ with entries $((2u\ln t)^{1/(1+\varepsilon)})_{t \in [T]}$.
\STATE $\hat S_O \leftarrow 0$, $n \leftarrow 0$
\FOR{$t \in [T]$}
\STATE $O_t \leftarrow \textsc{Observations}(t)$.
\FOR{$x_t \in O_t$}
\STATE $n \leftarrow n + 1$
\STATE $i_t \leftarrow \max\left(t, \textsc{Search}(B, (|x_t|^{1+\varepsilon}/n))\right)$.
\STATE $D(i_t) \leftarrow D(i_t) \cup \{x_t\}$.
\ENDFOR
\FOR{$x \in D(t)$}
\STATE $\hat S_O \leftarrow \hat S_O + x$.
\ENDFOR
\STATE $\hat\mu_O(t) \leftarrow \hat S_O/n$.
\ENDFOR
\end{algorithmic}
\end{algorithm}

Algorithm~\ref{alg:online_mean} assumes that for any $t$, a new set of observations $O_t$ is available, which it incorporates into the robust mean with $O(\ln t)$ per sample (instead of typically recomputing the mean for each $t$). The complexity stems from the binary search, assuming the dictionary lookup is $O(1)$.
\section{Experiments}
Our primary contributions are in leveraging cooperation to accelerate overall decision-making, and the most interesting aspects of this study pertain to how graph structures, scalability, heavy tails and decentralized vs. centralized estimation affect the group regret. To this end, we analyse these aspects in our experimental setup, and relegate other comparisons ($\Delta_k$, number of arms, etc.) to the appendix.

\textit{Reward Distributions}. We conduct experiments using $\alpha$-stable densities~\cite{levy1925calcul}, that admit finite moments only of order $< \alpha \leq 2$, and we consider $\alpha$-stable densities where $\alpha \geq 1$. The $\alpha$-stable family includes several widely used distributions, such as Gaussian ($\alpha=2$, only light-tailed density), L\'evy ($\alpha=0.5$) and Cauchy ($\alpha$=1). The primary advantage of this density is that $\alpha$ can be adjusted to alter the heaviness of the reward distribution ($\alpha > 1$).

\textit{Graph Partitioning}. For Algorithm~\ref{alg:ftl_mp_ucb}, we require computing the maximal weighted independent set of $\mathcal G$. This problem is NP-Hard for arbitrary $\mathcal G$, and difficult to approximate. We use the approximate algorithm presented in~\cite{lucas2014ising} that uses the QUBO~\cite{glover2018tutorial} solver.

\textbf{Experiment 1: Random Graphs}. We set $K=5$, $\alpha=1.9$ for the standard $\alpha$-stable density, and sample arm means randomly from the interval $[0, 1]$ for each arm every experiment. We then construct random graphs on 200 agents from the Erdos-Renyi (ER) ($p=0.7$) and Barabasi-Albert (BA) ($m=5$) random graph families, and compare all three of our algorithms (using the trimmed mean estimator, with $\gamma = \text{diam}(\mathcal G)/2$) with the \textsc{Consensus-UCB} and single-agent \textsc{Robust-UCB}\cite{bubeck2013bandits} algorithms. We compare the group regret $R_\mathcal G(T)$ vs. $T$, averaged over 100 random graphs and bandit instances. The results for Erdos-Renyi graphs (Figure~(\ref{fig:main}A)) and Barabasi-Albert graphs (Figure~(\ref{fig:main}B)) demonstrate that while our algorithms outperform the baselines (in the order dictated by regret bounds), the gain is larger for the former. We attribute this to the network connectivity, i.e., since Barabasi-Albert graphs have ``hubs'', the clique number $\bar\chi(\mathcal G)$ for these graphs is larger.

\textbf{Experiment 2: Real-World Networks}. We select the \textsf{p2p-Gnutella04} (Figure~\ref{fig:main}C) and \textsf{ego-Facebook} (Figure~\ref{fig:main}D) network structures from the SNAP repository~\cite{leskovec2016snap} to experiment with in the real-world setting. For both experiments, we sample subgraphs of 500 nodes, and use these subgraphs. A common misconception is to compare our distributed \textit{multi-agent} problem with the \textit{social network clustering} problem~\cite{gentile2014online, li2016collaborative}, which is more scalable since it is \textit{single-agent} (i.e., one action chosen per trial). These networks are chosen because they represent two diverse situations cooperative decision-making can be applicable in -- social networks and peer-to-peer communication networks. In both cases, we observe a similar trend. The gains are larger in the p2p-Gnutella case since ego-Facebook is dense (with fewer nodes), hence \textsc{Consensus-UCB} performs better as well.

\textbf{Experiment 3: Effect of $\gamma$ and $\alpha$}. As ablation experiments, we investigate the effect of communication density $\gamma$ (Figure~\ref{fig:main}E) and tail parameter $\alpha$ (Figure~\ref{fig:main}F) on the group regret.  For both experiments, we construct random graphs on 200 agents from the Erdos-Renyi ($p=0.7$) family. We compare the group regret at $T=10000$ trials as a function of $\gamma$, and $\alpha$, respectively. First, we observe that communication density has a significant effect on all but the \textsc{KMT-UCB} algorithms. Next, we see that \textsc{Consensus-UCB} progressively gets worse as the tail gets heavier (i.e, $\alpha \rightarrow 1^+$).

\section{Conclusion}
In this paper, we presented a treatment of cooperative bandit estimation under heavy tails. We provided the first asymptotic lower bound on cooperative estimation that holds for arbitrary graphs $\mathcal G$ and a wide variety of communication protocols. We present the first robust cooperative estimation algorithms that can all provide optimal regret, even without knowledge of $\mathcal G$. We support our bounds via experiments over random graphs as well. However, our work leaves several open questions in \textit{robust} multi-agent decision-making. 

First, we note that our best algorithm provides an asymptotic group regret of $O(\alpha(\mathcal G_\gamma)K\ln T)$, which is similar to the results obtained in the non-stochastic case~\cite{cesa2019delay, martinez2018decentralized}. The $\alpha(\mathcal G_\gamma)$ overhead can be attributed to the fact that information does not flow completely through the network (cf. Assumption~\ref{ass:comm_protocol}a). This leads us to believe that tighter lower bounds can be obtained by taking this aspect of the communication protocol into account. Moreover, in realistic settings, messages incur stochasticity, i.e. they can be dropped at random, or propagate with varying delay $\gamma$. This line of work has been studied in the single-agent setting~\cite{pike2017bandits, vernade2018contextual}, however the problem becomes more challenging when multiple agents interact simultaneously. 

The extension of our setting to the contextual case is not trivial. Robust single-agent estimation for linear bandits is a difficult problem from both the algorithmic and computational point of view, since statistically optimal multivariate estimators require exponential time to compute~\cite{lugosi2019robust}. Furthermore, delay creates a $\sqrt{\gamma}$ scaling of the regret~\cite{neu2010online}, which is amplified in the multi-agent setting. Addressing such scenarios is a difficult but crucial next step in this line of research.
\newpage
\onecolumn
\appendix 
\section{Appendix}
\subsection{Lower Bound}
The lower bound proceeds in a manner similar to the lower bound achieved in the single-agent case. We first state a few intermediary lemmas and restate the assumptions on the communication protocol.
 \begin{assumption}[Rate Assumption]
Let $X_1, ..., X_n$ be $n$ samples of an $\varepsilon$-heavy tailed random variable, where $\varepsilon \in (0, 1]$, and $\mathbb E[X] = \mu$. For positive constants $c, \rho$ suppose that there exists a robust estimator $\hat\mu(\delta, n)$ such that, with probability at least $1-\delta$,
\begin{equation*}
    |\hat\mu(\delta, n) - \mu| \leq 2 \rho^{\frac{1}{1+\varepsilon}}\left(\frac{c\log(\delta^{-1})}{n}\right)^{\frac{\varepsilon}{1+\varepsilon}}.
\end{equation*}
\label{ass:robust_mean_estimator}
\end{assumption}
 \begin{assumption}[Communication Protocol]
The communication protocol considered follows:
\begin{enumerate}
    \item Any agent $m$ is capable of sending a message $\bm q_{m}(t)$ to any other agent $m' \in [M]$, which is earliest received at time $t+\min(0, d(m, m') - 1)$.
    \item $\bm q_{m}(t)$ is a function of the action-reward pairs of agent $m$, i.e. $\bm q_{m}(t) = {\bm F}_t(A_{m, 1}, X_{m, 1}, ..., A_{m, t}, X_{m, t})$ for any deterministic, bijective and differentiable set of functions ${\bm F}_t = (f_{i, t})_{i \in [L]}, f_{i,t}: \mathbb R^{2t} \rightarrow \mathbb R$.
    \item  ${\bm F}_t$ satisfies $|\text{det} \left({\bm J}_t\right)| = \Lambda(m, t).$ Here, ${\bm J}_t(\cdot)$ is the Jacobian of ${\bm F}_t$, and $\Lambda$ is only a function of $m$ and $t$.
\end{enumerate}
\label{ass:comm_protocol}
\end{assumption}

\begin{theorem}[Carath\'eodory's Extension Theorem]
Let $\left(\Omega_1, \mathcal F_1\right), ..., \left(\Omega_n, \mathcal F_n\right)$ be measurable spaces and $\bar{\mu} : \mathcal F_1 \times ... \times \mathcal F_n \rightarrow [0,1]$ be a function such that (a) $\bar\mu\left(\Omega_1 \times ... \times \Omega_n\right) = 1$, and (b) $\bar\mu\left(\cup_{k=1}^\infty A_k\right) = \sum_{k=1}^\infty \bar\mu(A_k)$ for all sequences of disjoint sets with $A_k \in \mathcal F_1 \times ... \times \mathcal F_n$. Let $\Omega = \Omega_1 \times ... \times \Omega_n$ and $\mathcal F =\sigma(\mathcal F_1 \times ... \times \mathcal F_n)$.

Then there exists a unique probability measure $\mu$ on $(\Omega, \mathcal F)$ such that $\mu$ agrees with $\bar\mu$ on $\mathcal F_1 \times ... \times \mathcal F_n$.
\label{thm:caratheodory}
\end{theorem}
\begin{theorem}[Multiagent Divergence Decomposition]
Let $\mathcal E = \prod_{k \in [K]} \mathcal M_k$ be a structured family of $K$-armed bandit problems and $\nu = (\nu_k)_{k \in [K]}, \nu' = (\nu'_k)_{k \in [K]} \in \mathcal E$ be two bandit problem instances. Then, for any decentralized policy $\Pi_t = (\pi_{m, t})_{m \in [M], t \in [T]}$ that uses a communication protocol satisfying Assumptions~\ref{ass:comm_protocol}, the following is true.
\begin{equation*}
    \mathbb E_{\nu\Pi}\left[\log\frac{d\mathbb P_{\nu\Pi}}{d\mathbb P_{\nu'\Pi}}\left(A_{1, 1}, X_{1, 1}, ..., A_{M, T}, X_{M, T}\right)\right] = \sum_{k=1}^K \mathbb E_{\nu\Pi}\left[N_k(T)\right]\mathbb D_{\mathsf {KL}}(\nu_{A_{k}}, \nu'_{A_{k}}).
\end{equation*}
Here, $\mathbb P_{\nu\Pi}$ and $\mathbb P_{\nu'\Pi}$ denote the product measures arising from the interaction of $\nu$ and $\nu'$ with $\Pi$, and $N_k(T)$ denotes the total number of pulls of arm $k$ across all $M$ agents at time $T$.
\label{thm:divergence_decomposition}
\end{theorem}
\begin{proof}
Consider an agent $m$, and let the $\gamma$-neighborhood of this agent be given by $\mathcal N_\gamma(m)$, such that $N = \left|\mathcal N_\gamma(m) \right|$. At any instant $t$, let the agents actions be denoted by $A_{m, t}$, and the associated outcome variable be $X_{m, t}$. We denote the set of action-reward pairs for the agent at time $t$ as $H_t^m = (A_{m, 1}, X_{m, 1}, ..., A_{m, t}, X_{m, t})$.

At any instant $t$, the agent also receives messages from neighboring agents, delayed by their distance in the graph $\mathcal G$. By our third assumption we can write each message $Z_{m, m'}(t)$ sent fom agent $m$ to agent $m'$ at time $t$ as the following, for some deterministic, bijective differentiable function(s) $(f_i)_{i \in [L]}$ where $L$ is the length of the message.
\begin{equation}
    Z_{m, m'}(t) = f_i\left(A_{m, 1}, X_{m, 1}, ..., A_{m, t}, X_{m, t}\right)
\end{equation}
For each $t \in [T]$, let $\Omega_{t} = \left([K] \times \mathbb R\right)^{Mt} \subset \mathbb R^{2Mt}$ and $\mathcal F_{t} = \mathfrak B(\Omega_{T})$. As is the case with single-agent bandits, we can define coordinate projections that govern each of the random variables $A^m_t, X^m_t \ \forall t, m$ by creating an ordering of all elements of $H_t$.
\begin{align}
    A_{m, t}\left(a_{1, 1}, x_{1, 1}, ..., a_{M, T}, x_{M, T}\right) = a_{m, t} \\
    X_{m, t}\left(a_{1, 1}, x_{1, 1}, ..., a_{M, T}, x_{M, T}\right) = x_{m, t}
\end{align}
By our assumption on the nature of messages, we can express the density of any message $z_{m, m'}(t)$ as the following. Let $\bm F = (f_i)_{i \in [L]}$.
\begin{align}
    p(z_{m, m'}(t)) &= p(\bm F^{-1}\left(z_{m, m'}(t)\right)\left|\textsf{det}\left(\bm J(z_{m, m'}(t)\right)\right| \\
  &= p(\bm F^{-1}\left(\bm F(a_{m, 1}, x_{m, 1}, ..., a_{m, t}, x_{m, t}))\right)\left|\textsf{det}\left(\bm J(z_{m, m'}(t)\right)\right| \\
  &= p(a_{m, 1}, x_{m, 1}, ..., a_{m, t}, x_{m, t})\left|\textsf{det}\left(\bm J(z_{m, m'}(t)\right)\right| \\
  &= p(a_{m, 1}, x_{m, 1}, ..., a_{m, t}, x_{m, t})\Lambda(m, m', t)
\end{align}
This primarily implies that each message is completely specified by the corresponding inputs. With this probability space $(\Omega_{T}, \mathcal F_{T})$ we can then define a \textbf{decentralized policy} as a sequence $(\Pi_t)_{t=1}^T$, where $\Pi_t = (\pi_{m, t})_{m \in [M]}$ is a probability kernel from $(\Omega_{t-1}, \mathcal F_{t-1})$ to $([K]^M, 2^{[K]^M})$.

We now require a valid measure that connects $\Pi = (\Pi_t)_{t=1}^T$ and $\nu = (\nu_k)_{k \in [K]} \in \mathcal E$. The measure we will define will be similar to that of the canonical bandit model, however, we have a few key differences. First, we must note that individual elements $\pi_{m, t}$ of $\Pi_t$ factorize differently based on $\mathcal G$. Additionally, conditioned on $H^m_t$, $X_{m, t}$ follows the law $\nu_{A_{m, t}}$, i.e. it only depends on the corresponding arm pulled by the agent $m$. Therefore the conditions on the measure can be listed as follows.
\begin{enumerate}
    \item The conditional distribution of $A^m_t$ given $\cup_{ m \in [M]}\left(H^{m}_t\right)$ is $\pi_{m, t}\left(\cdot | \left(H^m_{t-1}\cup_{ m' \in \mathcal N_\gamma(m)}\left(H^{m'}_{t-d(m, m')-1}\right)\right)\right)$ almost surely. This condition is justified by the fact that each individual policy for an agent $m$ can only be dependent on information in the $\gamma$-neighborhood of the agent, and that information takes $d(m, m')$ steps to reach agent $m$ from any other agent $m'$.
    \item The conditional distribution of $X_{m, t}$ given $\bigcup_{ m \in [M]}\left(H^{m}_t\right)$ is $\nu_{A_{m,t}}$ almost surely. 
\end{enumerate}
Let $\lambda$ be a $\sigma$-finite measure on $(\mathbb R, \mathfrak B(\mathbb R))$ for which $\nu_k$ is absolutely continuous with respect to $\lambda$ for all $k \in [K]$. Let $p_k = \frac{d\nu_k}{d\lambda}$ be the Radon-Nikodym derivative of $\nu_k$ with respect to $\lambda$, and $\rho$ be the counting measure over $\mathfrak B(\mathbb R)$, We can define the density $p_{\nu\Pi} : \Omega \rightarrow \mathbb R$ with respect to the product measure $(\rho \times \mathbb R)^{MT}$ as the following.
\begin{multline}
p_{\nu\Pi}\left(a_{1, 1}, x_{1, 1}, ..., a_{M, T}, x_{M, T}\right) =\\ \prod_{m \in [M]} \pi_{m, T}\left(a_{m, T} | a_{1, 1}, x_{1, 1}, ..., a_{M, T-1}, x_{M, T-1}\right)p\left(a_{1, 1}, x_{1, 1}, ..., a_{M, T-1}, x_{M, T-1}\right)\nu_{a_{m, T}}(x_{m,T})
\end{multline}
\begin{multline}
= \left(\prod_{m \in [M]}\prod_{t \in [T]} \pi_{m, t}\left(a_{m, t} | \bigcup_{m' \in \mathcal N_\gamma(m) \cup \{m\}}\left\{a_{m', 1}, x_{m', 1}, ... a_{m', t-d(m, m')}, x_{m', t-d(m, m')}\right\}\right)\right) \times\\ \left(\prod_{m' \in [M]}\prod_{t \in [T]} \nu_{a_{m, t}}(x_{m, t})\right) \times \left(\prod_{m \in [M]}\prod_{m' \in \mathcal N_\gamma(m)}\prod_{t \in [T-d(m, m')]}\Lambda(m, m', t)\right)
\end{multline}
It can be easily shown that $p_{\nu\Pi}$ is a valid density, and satisfies the two properties listed earlier. By Theorem~\ref{thm:caratheodory}, we know that such a distribution exists. Let the corresponding measure be denoted by $\mathbb P_{\nu\Pi}$. We now prove a version of the canonical divergence decomposition in the presence of additional observations. 

In addition to $\nu$, let $\nu' = (\nu'_k)_{k \in [K]} \in \mathcal E$ be the reward distributions associated with a separate $k$-armed bandit problem, and $\mathbb P_{\nu'\pi}$ denote the joint measure for $\nu'$ under the same policy $\pi$. Assume that $\mathbb D_{\textsf{KL}}(\nu_k, \nu'_k) < \infty$, for all $k \in [K]$. We then have,
\begin{equation}
    \log\frac{d\mathbb P_{\nu\Pi}}{d\mathbb P_{\nu'\Pi}}\left(a_{1, 1}, x_{1, 1}, ..., a_{M, T}, x_{M, T}\right) =\sum_{m \in [M]}\sum_{t \in [T]}\log\left(\frac{\nu_{a_{m, t}}(x_{m, t})}{\nu'_{a_{m, t}}(x_{m, t})}\right)
\end{equation}
This follows from the chain rule of Radon-Nikodym derivatives and the fact that (a) the policy terms cancel out by the definitions of $p_{\nu\Pi}$ and $p_{\nu'\Pi}$, and (b) communication terms $\Lambda$ cancel out since they are independent of $\nu$. Taking expectations and replacing $H_T = \bigcup_{ m \in [M]}\left(H^{m}_T\right)$, we have,
\begin{align}
    \mathbb E_{\nu\Pi}\left[\log\frac{d\mathbb P_{\nu\Pi}}{d\mathbb P_{\nu'\Pi}}\left(H_T\right)\right] = \sum_{m \in [M]}\sum_{t \in [T]}\mathbb E_{\nu\pi}\left[\log\left(\frac{\nu_{A_{m, t}}(X_{m, t})}{\nu'_{A_{m,t}}(X_{m, t})}\right)\right] .
    \label{eqn:kl_divergence_rn}
\end{align}
Additionally, we also know that, for all $t \in [T], m \in [M]$,
\begin{align}
    \mathbb E_{\nu\Pi}\left[\log\left(\frac{\nu_{A_{m, t}}(X_{m, t})}{\nu'_{A_{m,t}}(X_{m, t})}\right)\right] = \mathbb E_{\nu\Pi}\left[\mathbb E_{\nu\Pi}\left[\log\left(\frac{\nu_{A_{m, t}}(X_{m, t})}{\nu'_{A_{m,t}}(X_{m, t})}\right)\right]\Bigg|A_{m, t}\right] = \mathbb E_{\nu\Pi}\left[\mathbb D_{\mathsf {KL}}(\nu_{A_{m, t}}, \nu'_{A_{m, t}}) \right].
\end{align}
Replacing the above identity in Equation~(\ref{eqn:kl_divergence_rn}), we have,
\begin{align}
    \mathbb E_{\nu\Pi}\left[\log\frac{d\mathbb P_{\nu\pi}}{d\mathbb P_{\nu'\pi}}\left(H_T\right)\right] &= \sum_{m \in [M]}\sum_{t \in [T]} \mathbb E_{\nu\Pi}\left[\mathbb D_{\mathsf {KL}}(\nu_{A_{m, t}}, \nu'_{A_{m, t}})\right] \\
    &= \sum_{k=1}^K \sum_{m \in [M]}\sum_{t \in [T]} \mathbb E_{\nu\Pi}\left[\mathbbm{1}\left\{A_{m,t} = k \right\}\right]\mathbb D_{\mathsf {KL}}(\nu_{A_{k}}, \nu'_{A_{k}})\\
    &= \sum_{k=1}^K \mathbb E_{\nu\Pi}\left[N_k(T)\right]\mathbb D_{\mathsf {KL}}(\nu_{k}, \nu'_{k}).
\end{align}
\end{proof}
\begin{theorem}[Bretagnolle-Huber Inequality]
Let $P$ and $Q$ be probability measures on the same measurable set $(\Omega, \mathcal F)$ and let $A$ be an arbitrary event. Then,
\begin{equation*}
    \mathbb P_P(A) + \mathbb P_Q(A^c) \geq \frac{1}{2}\exp\left(-\mathbb D_{\mathsf{KL}}(P, Q)\right).
\end{equation*}
Here $A^c$ denotes the complement event.
\label{thm:bretagnolle_huber}
\end{theorem}
\begin{theorem}[Asymptotic Lower Bound for Cooperative Multi-Agent Regret]
Let $\mathcal E = \prod_{k \in [K]} \mathcal M_k$ be a family of unstructured bandit problems over $K$ arms and $\nu = (\nu_k)_{k \in [K]} \in \mathcal E$ be a bandit problem such that $\mu(\nu_k) = \mu_k, \mu^* = \arg\max_{k \in [K]}\mu_k$ and $\Delta_k = \mu^* - \mu_k$. For any consistent cooperative multi-agent policy $\Pi = (\Pi_t)_{t \in [T]}$ on $M$ agents communicating over network $\mathcal G$ the following lower bound on the group regret holds for any horizon $T$.
\begin{equation*}
  R_\mathcal G(T) \geq \left(\sum_{k:\Delta_k > 0} \frac{\Delta_k}{\mathbb D^{\mathsf {inf}}_k}\right)\ln T.
\end{equation*}
Here, $\mathbb D^{\mathsf{inf}}_k = \inf_{\nu' \in \mathcal M_k}\left\{\mathbb D_{\mathsf {KL}}(\nu, \nu') : \mu(\nu') > \mu^*\right\}$, and $\mathbb D_{\mathsf{KL}}(\cdot, \cdot)$ denotes the Kullback-Leibler divergence of the first distribution from the second probability distribution.
\label{thm:lower_bound_pulls_appx}
\end{theorem}
\begin{proof}
This proof follows the standard approach for single-agent consistent bandit algorithms. Consider any suboptimal arm $i$ and let $\delta > 0$ be arbitrary. Consider $\nu' = (\nu'_k)_{k \in [K]} \in \mathcal E$ such that $\mathbb D_{\mathsf {KL}}(\nu_{k}, \nu'_{k}) \leq \mathbb D^{\mathsf {inf}}_i + \delta$, and $\mu(\nu'_i) > \mu^*$, which exists by the definition of $\mathbb D^{\mathsf {inf}}_i$. By Theorems~\ref{thm:bretagnolle_huber} and~\ref{thm:divergence_decomposition} we have the following for any event $A$.
\begin{align}
    \mathbb P_{\nu\Pi}(A) + \mathbb P_{\nu'\Pi}(A^c) &\geq \frac{1}{2}\exp\left(-\mathbb E_{\nu\Pi}[N_i(T)](\mathbb D^{\mathsf {inf}}_i + \delta)\right) \\ \intertext{Let $R_\mathcal G = R_\mathcal G(T, \nu, \Pi)$ be the regret obtained by $\Pi$ on $\nu$ and $R'_\mathcal G = R_\mathcal G(T, \nu', \Pi)$ be the regret obtained by $\Pi$ on $\nu'$. By choosing $A = \{N_i(T) > T/2\}$, we have,}
    R_\mathcal G + R'_\mathcal G &\geq \frac{T}{2}\left(\mathbb P_{\nu\Pi}(A)\Delta_i + \mathbb P_{\nu'\Pi}(A^c)(\mu'_i - \mu^*)\right) \\
  &\geq \frac{T}{2}\min\{\Delta_i, \mu'_i - \mu^*\}\left(\mathbb P_{\nu\Pi}(A) + \mathbb P_{\nu'\Pi}(A^c)\right) \\
  &\geq \frac{T}{4}\min\{\Delta_i, \mu'_i - \mu^*\}\exp\left(-\mathbb E_{\nu\Pi}[N_i(T)](\mathbb D^{\mathsf {inf}}_i + \delta)\right)\\ \intertext{Rearranging and taking limit inferior, we have,}
  \liminf_{T \rightarrow \infty}\frac{\mathbb E_{\nu\Pi}[N_i(T)]}{
  \ln(T)} &\geq \frac{1}{\mathbb D^{\mathsf {inf}}_i + \delta}\liminf_{T \rightarrow \infty}\frac{\ln\left(\frac{T\min\{\Delta_i, \mu'_i - \mu^*\}}{R_\mathcal G + R'_\mathcal G}\right)}{\ln(T)} \\ 
  &\geq \frac{1}{\mathbb D^{\mathsf {inf}}_i + \delta}\left(1-\limsup_{T \rightarrow \infty}\frac{\ln\left(R_\mathcal G + R'_\mathcal G\right)}{\ln(T)}\right) \\ \intertext{Using the fact that $\Pi$ is consistent, we have for some constant $a>0$ and constant $C_a$,}
  &\geq \frac{1}{\mathbb D^{\mathsf {inf}}_i + \delta}\left(1-\limsup_{T \rightarrow \infty}\frac{a\ln T + \ln(C_a)}{\ln(T)}\right).
\end{align}
Since $a > 0$ is arbitrary, taking the limit as $\delta$ goes to zero, we have, for any suboptimal arm $i$,
\begin{equation}
    \liminf_{T \rightarrow \infty}\frac{\mathbb E_{\nu\Pi}[N_i(T)]}{
  \ln(T)} \geq \frac{1}{\mathbb D^{\mathsf {inf}}_i}.
\end{equation}
Plugging this into the definition of regret and rearranging gets us the final result.
\end{proof}
\begin{corollary}[Lower Bound on Heavy-Tailed Cooperative Regret]
For any $\Delta \in (0, 1/4)$, there exist $K \geq 2$ distributions $\nu_1, ..., \nu_K$ satisfying
\begin{equation}
    \mathbb E_{X \sim \nu_k}[|X|^{1+\varepsilon}] \leq u, \text{ and } \mathbb E_{X \sim \nu_*}[X] - \mathbb E_{X \sim \nu_k}[X] = \Delta \forall k \in K,
\end{equation}
such that the following holds: Consider any consistent decentralized policy $\Pi_t = (\pi_{m, t})_{m \in [M], t \in [T]}$ that uses a communication protocol satisfying Assumptions~\ref{ass:comm_protocol}, on a bandit problem with reward distributions $\nu_1, ..., \nu_K$. Then, the policy satisfies,
\begin{equation*}
    R(T) \geq \left(\sum_{k=1}^K \frac{2^{1-\frac{1}{\varepsilon}}}{\Delta^{\frac{1}{\varepsilon}}} \right)\ln T.
\end{equation*}
Furthermore, there exists a set of $K$ distributions such that any lossless centralized algorithm over $M \leq K$ agents satisfying, for any suboptimal arm $k$, $\mathbb E[n_k(T)] = o(n^a)$ for $a > 0$ obtains regret of order  $\Omega\left(K^{\frac{\varepsilon}{1+\varepsilon}}T^{\frac{1}{1+\varepsilon}}\right)$.
\end{corollary}
\begin{proof}
The proof proceeds in the standard construction outlined in~\cite{bubeck2013bandits}. Consider $\nu_1(x) = \left(1 - \alpha^{1+\varepsilon}\right) \delta(x) + \alpha^{1+\varepsilon}\delta(x - 1/\alpha)$, where $\alpha = \left(2\Delta\right)^{\frac{1}{\varepsilon}}$, $\delta(x - \gamma)$ is the Dirac distribution at $\gamma$, and let $\forall \ k \in [2, K], \nu_k(x) = \left(1 - \alpha^{1+\varepsilon} + \Delta\alpha\right)\delta(x) + \left(\alpha^{1+\varepsilon} - \Delta\alpha \right)\delta(x - 1/\alpha)$. We can see that
\begin{equation*}
    \mathbb E_{X \sim \nu_i}[|X|^{1+\varepsilon}] = 1 \ \forall i \in [K] \text{ and } \mathbb E_{X \sim \nu_1}[X] - \mathbb E_{X \sim \nu_i}[X] = \Delta \ \forall i \in [2, K].
\end{equation*}
Hence, $\nu_i$ satisfy the constraints stated in the Theorem. Now, we can see that $\nu_1$ corresponds to a scaled Bernoulli distribution with parameter $\alpha^{1+\varepsilon}$, and similarly, $\nu_i, i \in [2, K]$ correspond to a scaled Bernoulli distribution with parameter $\alpha^{1+\varepsilon} - \Delta\alpha$. Since an algorithm operating on $\nu_1, ..., \nu_K$ will exhibit identical behavior on reward distributions $\mathcal{B}(\alpha^{1+\varepsilon}), ..., \mathcal{B}(\alpha^{1+\varepsilon} - \Delta\alpha)$, we can note the following for Bernoulli distributions for two Bernoulli distributions with arm parameter $\mu_k$ optimal parameter $\mu^*$.
\begin{align}
    \mathbb D^{\mathsf {inf}}_k = \mu_k\log\left(\frac{\mu_k}{\mu_*}\right) + (1-\mu)\log\left(\frac{1-\mu_k}{1-\mu_*}\right) = \mathbb D_{\mathsf {KL}}(\mu_k, \mu^*)
\end{align}
Therefore we can apply Theorem~\ref{thm:lower_bound_pulls_appx} directly to obtain the following lower bound on the number of pulls of any suboptimal arm.
\begin{align}
    \mathbb E[N_k(T)] &\geq \left(\frac{1}{\mathbb D_{\mathsf{KL}}(\nu_k, \nu_*)}\right)\ln T \\
    & \stackrel{(a)}{=} \left(\frac{1}{\mathbb D_{\mathsf{KL}}(  \mathcal{B}(\alpha^{1+\varepsilon} - \Delta\alpha), \mathcal{B}(\alpha^{1+\varepsilon}))}\right)\ln T \\
    & \stackrel{(b)}{\geq} \left(\frac{\alpha^{\varepsilon-1}-\alpha^{2\varepsilon}}{\Delta^2}\right)\ln T \\
    & = \left(\frac{2^{1-\frac{1}{\varepsilon}}}{\Delta^{1+\frac{1}{\varepsilon}}} \right)\ln T.
\end{align}
Here, $(a)$ is obtained by the equivalence of $\nu_1, ..., \nu_K$ to Bernoulli distributions, and $(b)$ is obtained by the inequality $\mathbb D_{\mathsf{KL}}(\mathcal B(\theta_2), \mathcal B(\theta_1)) \leq \frac{(\theta_1 - \theta_2)^2}{\theta_1(1- \theta_1)}$. We now decompose the regret.
\begin{align}
    R(T) &= \sum_{k=2}^K \Delta_k \mathbb E[N_k(T)] \\
    &\geq \sum_{k=2}^K \left(\frac{2^{1-\frac{1}{\varepsilon}}}{\Delta^{\frac{1}{\varepsilon}}} \right)\ln T.
\end{align}
For the problem-independent version of the bound, we can set $\Delta = \left(\frac{K}{T}\right)^{\frac{\varepsilon}{1+\varepsilon}}$ to obtain the desired result.
\end{proof}
\subsection{Regret of the \textsc{Consensus-UCB} Algorithm}
\begin{proposition}[Polynomial Concentration of Empirical Mean~\cite{bubeck2013bandits}]
Let $X_1, ..., X_n$ be $n$ copies of a random variable $X$ such that $\mathbb E[X]=\mu$ and $\mathbb E[|X-\mu|^{1+\varepsilon}] \leq v, \mathbb E[|X|^{1+\varepsilon}] \leq u$ for some $v > 0, u>0, \varepsilon \in (0, 1]$. Then, with probability at least $1-\delta, \delta \in (0, 1)$, we have
\begin{equation*}
    \left|\frac{1}{n}\sum_{i=1}^n X_i - \mu\right| \leq \left(\frac{6v}{\delta n^\varepsilon}\right)^{\frac{1}{1+\varepsilon}}.
\end{equation*}
This bound is tight up to a constant factor.
\label{prop:poly_conc_mean}
\end{proposition}
\begin{proposition}[Moment Estimation in Consensus Algorithms~\cite{landgren2016distributed}]
For $M$ agents communicating on graph $\mathcal G$, let $\bm L$ and $\bm A$ denote the graph Laplacian and adjacency matrix of $\mathcal G$, and $\bm P = \bm I_M - \kappa\cdot d_{\max}^{-1}\bm L$ is a row stochastic matrix, where $\bm I_M$ is the identity matrix of order $M$, $\kappa > 0$ is a constant and $d_{\max} = \max_{m \in \mathcal G} \textsf{deg}(m)$. Assume WLOG that the eigenvalues $\lambda_i$ of $\bm P$ are ordered such as $\lambda_1 =1 \geq ... \geq \lambda_M > -1$, and the corresponding eigenvectors are denoted by ${\bm u}_1, ..., {\bm u}_M$. Define
\begin{align}
    \epsilon &= \sqrt{M}\sum_{p=2}^M \frac{|\lambda_p|}{1-|\lambda_p|}, \\
    \omega^{+}_{pj} &= \sum_{d=1}^M u^d_pu^d_j\mathbbm{1}\left\{ (\bm u_p \bm u_j^\top)_{kk} \geq 0\right\}, \\
    \omega^{-}_{pj} &= \sum_{d=1}^M u^d_pu^d_j\mathbbm{1}\left\{ (\bm u_p \bm u_j^\top)_{kk} \leq 0\right\}, \\
    a_{ij}(k) &= \begin{cases}
         \omega^{+}_{ij}(\bm u_p \bm u_j^\top)_{kk} & \text{if } \lambda_p\lambda_j \geq 0 \text{ and } (\bm u_p \bm u_j^\top)_{kk} \geq 0, \\
         \omega^{-}_{ij}(\bm u_p \bm u_j^\top)_{kk} & \text{if } \lambda_p\lambda_j \geq 0 \text{ and } (\bm u_p \bm u_j^\top)_{kk} \leq 0, \\
         \omega_{ij}|(\bm u_p \bm u_j^\top)_{kk}| & \text{if } \lambda_p\lambda_j \leq 0, \\
    \end{cases}\\
    \epsilon^k &= M \sum_{p=1}^M \sum_{j=2}^M \frac{|\lambda_p\lambda_j|}{1-|\lambda_p\lambda_j|}a_{pj}(k).
\end{align}
Then, for rewards from arm $k \in [K]$ are drawn from an $\varepsilon$-heavy tailed distribution $\nu_k$ with mean $\mu_k$, and known upper bounds on the $(1+\varepsilon)^{th}$ moments $\mathbb E_{X \sim \nu_k}[|X-\mu_k|^{1+\varepsilon}] \leq v$ and $\mathbb E_{X \sim \nu_k}[|X|^{1+\varepsilon}] \leq u$, the following is true.
\begin{enumerate}
    \item The estimate $\hat{n}^m_k(t)$ satisfies
    \begin{equation*}
        N_k(t) - \epsilon \leq \hat{n}^m_k(t) \leq N_k(t) + \epsilon.
    \end{equation*}
    \item The estimator $\hat\mu_k^m(t)$ is unbiased.
    \item The variance of $\hat\mu_k^m(t)$ satisfies
    \begin{equation*}
        \mathbb E[|\hat\mu_k^m(t) - \mu_k|^{2}] \leq \frac{v}{M}\frac{\hat{n}^m_k(t) + \epsilon^k}{\hat{n}^m_k(t)^2}.
    \end{equation*}
\end{enumerate}
\label{prop:consensus_protocol}
\end{proposition}
We defer the readers to~\cite{landgren2016distributed} for a complete proof of this statement. Using these two results, we can derive the regret bound for the \textsc{Consensus-UCB} algorithm presented earlier.
\begin{theorem}
Consider the multi-agent co-operative bandit problem with graph $\mathcal G$ (with adjacency matrix $A$) on $M$ agents over $K$ arms with rewards drawn from heavy-tailed distributions, each with maximum admissible moment of order $1+\varepsilon$ and mean $\mu_k$. The \textsc{Consensus-UCB} algorithm obtains a total regret after $T$ iterations that follows the following bound.
\begin{equation*}
  R_\mathcal G(T) \leq \sum_{k=1}^K \Delta_k \left(M + M\epsilon + T^{2/3}\left(6v\left(\frac{1 + \epsilon^k}{\Delta_k^2}\right) + M\right)\right).
\end{equation*}
Here, $\epsilon^k =\sum_{p=1}^M \sum_{j=2}^M \frac{|\lambda_p\lambda_j|}{1-|\lambda_p\lambda_j|}a_{pj}(k)$ and $\epsilon = \sqrt{M}\sum_{p=2}^M \frac{|\lambda_p|}{1-|\lambda_p|}$, with $a_{pj}(k)$ as defined in Proposition~\ref{prop:consensus_protocol}.
\end{theorem}
\begin{proof}
Consider any agent $m \in [M]$, and its associated regret $R_m(T)$ at time $T$. 
\begin{align}
    R_m(T) &= \sum_{k=1}^K \Delta_k \sum_{t=1}^T \mathbb P\left(A_{m, t} = k\right) \\ \intertext{For any set of constants $\eta_m^k \geq 0$, we have,}
    &\leq \sum_{k=1}^K \Delta_k \left(\eta_m^k + \sum_{t>\eta_m^k}^T \mathbb P\left(A_{m, t} = k, N_k(t) \geq \eta_m^k \right)\right)
\end{align}
We now examine when an arm $k$ is pulled by agent $m$. The arm is pulled when one of three events occurs (for all $\alpha \in (0, 1]$).
\begin{align}
    \text{Event (A): } &  \hat\mu_*^m(t-1) \leq \mu_* - \sqrt{\frac{6vt^\alpha}{M}\left(\frac{\hat{n}^m_*(t) + \epsilon^*}{\hat{n}^m_*(t)^2}\right)}. \\
    \text{Event (B): } &  \hat\mu_k^m(t-1) \geq \mu_k + \sqrt{\frac{6vt^\alpha}{M}\left(\frac{\hat{n}^m_k(t) + \epsilon^k}{\hat{n}^m_k(t)^2}\right)}.\\
    \text{Event (C): } &  \mu^* \leq \mu^k + 2\sqrt{\frac{6vt^\alpha}{M}\left(\frac{\hat{n}^m_k(t) + \epsilon^k}{\hat{n}^m_k(t)^2}\right)}.
\end{align}
From Proposition~\ref{prop:poly_conc_mean}, we know that by setting $\delta = t^{-\alpha}$, the probability of Event (A) or (B) occuring is,
\begin{equation}
    \mathbb P\left(\text{Event (A) or (B) occurs}\right) = 2\sum_{s=1}^T \frac{1}{s^{2\alpha}} \leq 2t^{1-2\alpha}. 
\end{equation}
Now, we see that for Event (C) to occur, the following must be true\cite{landgren2016distributed}.
\begin{align}
    \hat n_k(t) &< \ceil*{\epsilon + \frac{6vt^\alpha}{M}\left(\frac{\hat{n}^m_k(t) + \epsilon^k}{\Delta_k^2\hat{n}^m_k(t)^2}\right)} \\
    &\leq \ceil*{\epsilon + \frac{6vt^\alpha}{M}\left(\frac{1 + \epsilon^k}{\Delta_k^2}\right)} .
\end{align}
It then follows that Event (C) does not occur when $\hat n_k(t) \geq \ceil*{\epsilon + \frac{6vT^\alpha}{M}\left(\frac{1 + \epsilon^k}{\Delta_k^2}\right)}$. Hence, we can set $\eta_m^k = 1 + \epsilon + \frac{6vT^\alpha}{M}\left(\frac{1 + \epsilon^k}{\Delta_k^2}\right)$. Then, the regret can be simplified as the following.
\begin{align}
    R_m(T) &\leq \sum_{k=1}^K \Delta_k \left(1 + \epsilon + \frac{6vT^\alpha}{M}\left(\frac{1 + \epsilon^k}{\Delta_k^2}\right) + 2\sum_{t=1}^T t^{1-2\alpha}\right) \\
    &\leq \sum_{k=1}^K \Delta_k \left(1 + \epsilon + \frac{6vT^\alpha}{M}\left(\frac{1 + \epsilon^k}{\Delta_k^2}\right) + T^{2-2\alpha}\right) \\ \intertext{Setting $\alpha = 2/3$, we have,}
    &= \sum_{k=1}^K \Delta_k \left(1 + \epsilon + T^{2/3}\left(\frac{6v}{M}\left(\frac{1 + \epsilon^k}{\Delta_k^2}\right) + 1\right)\right).
\end{align}
Summing over all agents $m \in [M]$, we get the final result.
\begin{align}
    R_\mathcal G(T) \leq \sum_{k=1}^K \Delta_k \left(M + M\epsilon + T^{2/3}\left(6v\left(\frac{1 + \epsilon^k}{\Delta_k^2}\right) + M\right)\right).
\end{align}
\end{proof}

\subsection{Concentration Results for Robust Mean Estimators}
Here we restate the concentration results for robust estimators for heavy-tailed distributions. We refer the readers to~\cite{bubeck2013bandits} for full proofs.
\begin{lemma}[Online Trimmed Mean Estimator\cite{bubeck2013bandits}]
Consider $n$ copies $X_1, ..., X_n$ of a heavy-tailed random variable $X$ such that $\mathbb E[X] = \mu, \mathbb E[X^{1+\varepsilon}] \leq u$ for some $\varepsilon \in (0, 1]$. The online trimmed mean, for some $\delta \in (0, 1)$ is defined as
\begin{equation*}
    \hat{\mu}_O = \frac{1}{n}\sum_{i=1}^n X_i \mathbbm{1}\left\{|X_i| \leq \left(\frac{ui}{\log \delta^{-1}}\right)^{\frac{1}{1+\varepsilon}}\right\}.
\end{equation*}
With probability at least $1-\delta$, the following lower bound is true.
\begin{equation*}
    \hat\mu_O \leq \mu + 4u^{\frac{1}{1+\varepsilon}}\left(\frac{\log(\delta^{-1})}{n}\right)^{\frac{\varepsilon}{1+\varepsilon}}.
\end{equation*}
\end{lemma}

\begin{lemma}[Median-of-Means Estimator\cite{bubeck2013bandits}]Let $X_1, ..., X_n$ be copies of a heavy-tailed random variable $X$ such that $\mathbb E[X] =\mu$ and $\mathbb E[|X-\mu|^{1+\varepsilon}] \leq u$ for some $\varepsilon \in (0, 1]$. Let $k = \floor*{8\log(e^{1/8}/\delta) \land n/2}$ and $N = \floor*{n/k}$. Let the group-wise means be given by the following.
\begin{equation*}
    \hat\mu_1 = \frac{1}{N}\sum_{i=1}^N X_i, ..., \hat\mu_k = \frac{1}{N}\sum_{i=(k-1)N+1}^{kN} X_i.
\end{equation*}
Let $\hat\mu_M$ denote the median of $\hat\mu_1, ..., \hat\mu_k$. Then with probability at least $1-\delta$,
\begin{equation*}
    \hat\mu_M \leq \mu + (12v)^{\frac{1}{1+\varepsilon}}\left(\frac{16\log(e^{1/8}/\delta)}{n}\right)^{\frac{\varepsilon}{1+\varepsilon}}.
\end{equation*}
\end{lemma}
\begin{lemma}[Catoni's estimator~\cite{catoni2012challenging}]Let $X_1, ..., X_n$ be copies of a heavy-tailed random variable $X$ such that $\mathbb E[X] =\mu$ and $\mathbb E[|X-\mu|^{2}] \leq u$. $\psi : \mathbb R \rightarrow \mathbb R$ be a continuous strictly increasing function such that
\begin{equation*}
    -\log(1-x+x^2/2) \leq \psi(x) \leq \log(1+x+x^2/2).
\end{equation*}
Let $\alpha_\delta = \sqrt{\frac{2\log(1/\delta}{n(v + \frac{2v\log(1/\delta)}{n-2\log(1/\delta}}}$. Let $\hat\mu_C$ be the unique value that satisfies the following equation.
\begin{equation*}
    \sum_{i=1}^n \psi(\alpha_\delta(X_i - \hat\mu_C) = 0.
\end{equation*}
Then, with probability at least $1-\delta$,
\begin{equation*}
    \hat\mu_C \leq \mu + 2\sqrt{\frac{v\log(\delta^{-1})}{n}}.
\end{equation*}
\end{lemma}
\subsection{Proofs for Message-Passing Protocol}
\begin{lemma}
For graph $\mathcal G$ and agent $m$, consider a subgraph $\mathcal G_\gamma(m)$ that includes all agents that have a shortest path of length at most $\gamma$ from agent $m$, along with the corresponding paths. Let $S^m_k(t)$ denote the set of all reward samples (across all agents) possessed by agent $m$ for arm $k$ at time $t$, and $N_{\mathcal G_\gamma(m)}^k(t)$ denote the total number of times arm $k$ has been pulled until time $t$ across all agents in $\mathcal G_\gamma(m)$. Then, we have, for all $k \in [K], m \in [M]$,
\begin{align*}
    N_{\mathcal G_\gamma(m)}^k(t) \geq \left|S^m_k(t)\right| \geq \max\left\{0, N_{\mathcal G_\gamma(m)}^k(t) + (M-1) (1-\gamma)\right\}.
\end{align*}
\label{lemma:broadcast_delay}
\end{lemma}
\begin{proof}
Let $S^m_k(t)$ denote the set of all reward samples (across all agents) possessed by agent $m$ for arm $k$ at time $t$. Similarly, let $P_m^k(t)$ denote the set of reward samples obtained by agent $m$ for its own pulls of arm $k$ until time $t$. We know, then that $P_m^k(t) = P_m^k(t-1) \text{ if arm $k$ was pulled at time $t$, and }P_m^k(t) = P_m^k(t-1) \cup \{X_{m, t}\} \text{ otherwise.}$

Additionally, any message from an agent $m' \in \mathcal G$ takes $d(m, m') - 1$ iterations to reach agent $m$. Therefore:
\begin{align}
    S^m_k(t) = P_m^k(t) \cup \left\{ \bigcup_{m' \in \mathcal G \setminus \{m\}} P_m^k\left(t-d(m',m)+1\right)\right\}.
\end{align}
Note that $P_m^k(t)$ and $P_{m'}^{k'}(t')$ are disjoint for all $m \neq m', k, k', t, t'$. Let $n(S)$ denote the cardinality of $S$. Then,
\begin{align}
    n\left(S^m_k(t)\right) = n\left(P_m^k(t)\right) + \left\{ \sum_{m' \in \mathcal G \setminus \{m\}} n\left(P_m^k\left(t-d(m',m)+1\right)\right)\right\}.
\end{align}
Now, in the iterations $t-d(m, m')+1$ to $t$, agent $m'$ can pull arm $k$ at most $d(m, m') - 1$ times and at least 0 times. Therefore,
\begin{align}
    N_{\mathcal G_\gamma(m)}^k(t) \geq n\left(S^m_k(t)\right) &\geq \max\left\{0, N_{\mathcal G_\gamma(m)}^k(t) - \sum_{m' \in \mathcal G_\gamma(m) \setminus \{m\}}(d(m, m')-1)\right\}\\
    &\geq \max\left\{0, N_{\mathcal G_\gamma(m)}^k(t) + |\mathcal G_\gamma(m)|(1-\gamma)\right\}
\end{align}
The final statement is obtained by the fact that $|\mathcal G_\gamma(m)| \leq M-1$.
\end{proof}
\begin{theorem}
\textsc{Decentralized MP-UCB} with message life parameter $\gamma$ and mean estimator $\hat\mu_R(n, \delta)$ that satisfies Assumption~\ref{ass:robust_mean_estimator} with constants $c$ and $v$ obtains a total regret after $T$ iterations that follows the following bound.
\begin{align*}
    R_\mathcal G(T) \leq 4Cv^{\frac{1}{\varepsilon}}\chi\left(\overline{\mathcal G}_\gamma\right)\left(\sum_{k : \Delta_k > 0} (2\Delta_k)^{-1/\varepsilon}\right)\ln T + \left(3M + \gamma\chi\left(\bar{\mathcal G_\gamma}\right)\left(M-1\right)\right)\left(\sum_{k : \Delta_k > 0}\Delta_k\right).
\end{align*}
Here, $C$ is a positive constant independent of $T, K, M$, and $\chi(\cdot)$ refers to the chromatic number.
\end{theorem}
\begin{proof}
It can be seen that $A_\gamma$ hence contains edges between two vertices $i$ and $j$ if there exists at least one path of length at most $\gamma$ between them. Let a clique covering of $\mathcal G_\gamma$ be given by $\bm{C}_\gamma$. We begin by decomposing the group regret.
\begin{align}
    R_{\mathcal G}(T) &= \sum_{m=1}^M R_m(T) \\
    &\leq \sum_{\mathcal C \in \bm{C}_\gamma} \sum_{m \in \mathcal C} \sum_{k=1}^K \Delta_{k}\mathbb E[n_m^k(T)] \\
    &= \sum_{\mathcal C \in \bm{C}_\gamma}\sum_{k=1}^K \Delta_k \left(\sum_{m\in\mathcal C}\sum_{t=1}^T \mathbb P\left(A_{m, t} = k\right)\right) \\ \intertext{Consider the cumulative regret $R_{\mathcal C}(T)$ within the clique $\mathcal C$. For some time $T^k_{\mathcal C}$, assume that each agent has pulled arm $k$ for $\eta_m^k$ trials, where $\eta_{\mathcal C}^k = \sum_{m \in \mathcal C}\eta_m^k$. Then,}
    R_{\mathcal C}(T)&\leq \sum_{k=1}^K \Delta_k \left(\eta_{\mathcal C}^k + \sum_{m \in \mathcal C} \sum_{t=T^k_\mathcal C}^T \mathbb P\left(A_{m, t} = k, N_k^\mathcal C(t) \geq \eta^k_\mathcal C \right)\right). \label{eqn:regret_bare}
\end{align}
Here $N_k^{\mathcal C}(t)$ denotes the number of times arm $k$ has been pulled by any agent in $\mathcal C$. We now examine the probability of agent $m \in \mathcal C$ pulling arm $k$. Note that an arm is pulled when one of three events occurs:
\begin{align}
    \text{Event (A): } &  \hat\mu_*^m(t-1) \leq \mu_* - v^{\frac{1}{1+\varepsilon}}\left(\frac{2c\ln t}{|S_m^*(t)|}\right)^{\frac{\varepsilon}{1+\varepsilon}} \\
    \text{Event (B): } &  \hat\mu_k^m(t-1) \geq \mu_k + v^{\frac{1}{1+\varepsilon}}\left(\frac{2c\ln t}{|S^m_k(t)|}\right)^{\frac{\varepsilon}{1+\varepsilon}} \\
    \text{Event (C): } &  \mu^* \leq \mu^k + 2v^{\frac{1}{1+\varepsilon}}\left(\frac{2c\ln t}{|S^m_k(t)|}\right)^{\frac{\varepsilon}{1+\varepsilon}} 
\end{align}
Now, let us examine the occurence of event $(C)$:
\begin{align}
  \Delta_k \leq 2v^{\frac{1}{1+\varepsilon}}\left(\frac{2c\ln t}{|S^m_k(t)|}\right)^{\frac{\varepsilon}{1+\varepsilon}} \\
  \implies |S^m_k(t)| \leq 2cv^{\frac{1}{\varepsilon}}\ln t\left(\frac{2}{\Delta_k}\right)^{1+\frac{1}{\varepsilon}} 
\end{align}
We know that for the subgraph $\mathcal C$, Lemma~\ref{lemma:broadcast_delay} holds for each $m \in \mathcal C$ with delay $\gamma$. Hence, $N_m^k(t) \geq N^{\mathcal C}_k(t) - (M - 1)(1-\gamma)$ for all $t$. Therefore, if we set $\eta^k_\mathcal C = \ceil*{2cv^{\frac{1}{\varepsilon}}\ln t\left(\frac{2}{\Delta_k}\right)^{1+\frac{1}{\varepsilon}} + (M - 1)(\gamma - 1)}$, we know that event $(C)$ will not occur. Additionally, using the union bound over $N^*_m(t)$ and $N^k_m(t)$, and Assumption~\ref{ass:robust_mean_estimator}, we have:
\begin{align}
    \mathbb P(\text{Event (A) or (B) occurs}) \leq 2\sum_{s=1}^t \frac{1}{s^4} \leq \frac{2}{t^3}.
\end{align}
Combining all probabilities, and inserting in Equation~(\ref{eqn:regret_bare}), we have,
\begin{align}
    R_{\mathcal C}(T)&\leq \sum_{k=1}^K \Delta_k \left(\eta_{\mathcal C}^k + \sum_{m \in \mathcal C} \sum_{t=T^k_\mathcal C}^T \mathbb P\left(A_{m, t} = k, N_k^\mathcal C(t) \geq \eta^k_\mathcal C \right)\right)\\
    &\leq \sum_{k=1}^K \Delta_k \left(\ceil*{2cv^{\frac{1}{\varepsilon}}\ln t\left(\frac{2}{\Delta_k}\right)^{1+\frac{1}{\varepsilon}} + (M - 1)(\gamma - 1)} + \sum_{m \in \mathcal C} \sum_{t=1}^T \frac{2}{t^3}\right) \\
    &\leq \sum_{k=1}^K \Delta_k \left(\ceil*{2cv^{\frac{1}{\varepsilon}}\ln t\left(\frac{2}{\Delta_k}\right)^{1+\frac{1}{\varepsilon}} + (M - 1)(\gamma - 1)} + 4|\mathcal C|\right) \\
    &\leq \sum_{k=1}^K \Delta_k \left(2cv^{\frac{1}{\varepsilon}}\ln t\left(\frac{2}{\Delta_k}\right)^{1+\frac{1}{\varepsilon}} + (M - 1)(\gamma - 1) + 1 + 4|\mathcal C|\right).
\end{align}
We can now substitute this result in the total regret.
\begin{align}
    R_\mathcal G(T) &\leq \sum_{\mathcal C \in \bm{C}} R_\mathcal C(T) \\
    &\leq \sum_{\mathcal C \in \bm{C}}\sum_{k=1}^K \Delta_k \left(2cv^{\frac{1}{\varepsilon}}\ln t\left(\frac{2}{\Delta_k}\right)^{1+\frac{1}{\varepsilon}} + (M - 1)(\gamma - 1) + 1 + 4|\mathcal C|\right) \\
    &= \sum_{k=1}^K \frac{4cv^{\frac{1}{\varepsilon}}\chi\left(\bar{\mathcal G_\gamma}\right)}{(2\Delta_k)^{1/\varepsilon}}\log T + \left(3M + \gamma\chi\left(\bar{\mathcal G_\gamma}\right)\left(M-1\right)\right)\left(\sum_{k=1}^K\Delta_k\right).
\end{align}
\end{proof}
\begin{theorem}
For any horizon $T > K$, and robust mean estimator that satisfies Assumption~\ref{ass:robust_mean_estimator} with constants $c$ and $v$, the group regret under \textsc{Centralized MP-UCB} is bounded by the following.
\begin{align*}
    R_\mathcal G(T) \leq \alpha(\mathcal G_\gamma)\left(\sum_{k:\Delta_k>0}\Delta_k^{-1/\varepsilon}\right)\left(2^{1+1/\varepsilon}cv^{1/\varepsilon}\right)\ln T + \left(\alpha(\mathcal G_\gamma)M\gamma + M + \alpha(\mathcal G_\gamma)\right)\left(\sum_{k: \Delta_k > 0} \Delta_k\right)
\end{align*}
For $T \leq K, R_\mathcal G(T) \leq MK$. $\alpha(\cdot)$ denotes the independence number.
\end{theorem}
\begin{proof}
We begin by decomposing the group regret. 
\begin{align}
    R_\mathcal G(T) &= \sum_{m \in \mathcal G} R_m(T) \\
    &= \sum_{m \in V'} \left(R_m(T) + \sum_{f \in F(m)} R_f(T)\right) \\
    &= \sum_{t=1}^T\sum_{k: \Delta_k > 0} \sum_{m \in V'} \Delta_k \left(\mathbb P\left\{A_{m, t} = k \right\} + \sum_{f \in F(m)}  \mathbb P\left\{a_f(t) = k \right\}\right) \\ \intertext{For constants $\eta_k^m, \eta_k^f > 0, m \in V', f\in F(m)$, let $U_m^k$ be the event when $\eta_k^m + \sum_{f \in F(m)}\eta_k^f \leq N_k^m(t) + \sum_{f \in F(m)} N_k^f(t)$. Then we have,}
    &= \sum_{k: \Delta_k > 0}\sum_{m \in V'} \Delta_k\left(\eta_k^m + \sum_{f \in F(m)}\eta_k^f +\sum_{t=1}^T \mathbb P\left\{A_{m, t} = k ; U_k^m\right\} + \sum_{f \in F(m)}\sum_{t=1}^T \mathbb P\left\{a_f(t) = k ; U_k^m\right\}\right) \\ \intertext{We know that $a_f(t) = a_m(t-d(f,m))$. Let $\beta_k^m = \eta_k^m + \sum_{f \in F(m)}\eta_k^f$ for brevity. Therefore,}
    &\leq \sum_{k: \Delta_k > 0}\sum_{m \in V'} \Delta_k\left(\beta_k^m +\sum_{t=1}^T \mathbb P\left\{A_{m, t} = k ; U_k^m\right\}  + \sum_{t=1}^{T-d(m,f)} \mathbb P\left\{A_{m, t} = k ; U_k^m\right\}\right)+ \sum_{m \in V'}\sum_{f \in F(m)}d(m,f) \\
    &\leq \sum_{k: \Delta_k > 0}\sum_{m \in V'} \Delta_k\left(\beta_k^m + (|F(m)|+1)\left(\sum_{t=1}^T \mathbb P\left\{A_{m, t} = k  ; U_k^m\right\}\right)\right)
\end{align}
We see that a suboptimal arm is pulled when one of three events occurs.
\begin{align}
    \text{Event (A): } &  \hat\mu_*^m(t-1) \leq \mu_* - v^{\frac{1}{1+\varepsilon}}\left(\frac{2c\ln t}{|S_m^*(t)|}\right)^{\frac{\varepsilon}{1+\varepsilon}} \\
    \text{Event (B): } &  \hat\mu_k^m(t-1) \geq \mu_k + v^{\frac{1}{1+\varepsilon}}\left(\frac{2c\ln t}{|S^m_k(t)|}\right)^{\frac{\varepsilon}{1+\varepsilon}} \\
    \text{Event (C): } &  \mu_* \leq \mu_k + 2v^{\frac{1}{1+\varepsilon}}\left(\frac{2c\ln t}{|S^m_k(t)|}\right)^{\frac{\varepsilon}{1+\varepsilon}} 
\end{align}
Now, let us examine the occurence of event $(C)$:
\begin{align}
  \Delta_k &\leq 2v^{\frac{1}{1+\varepsilon}}\left(\frac{2c\ln t}{|S^m_k(t)|}\right)^{\frac{\varepsilon}{1+\varepsilon}} \\
  \implies |S^m_k(t)| &\leq cv^{\frac{1}{\varepsilon}}\ln t\left(\frac{2}{\Delta_k}\right)^{1+\frac{1}{\varepsilon}} \\ \intertext{Since agent $m$ can communicate only with its neighborhood $\mathcal N_\gamma(m)$, and $F(m) \subseteq \mathcal N_\gamma(m)$. Therefore $\sum_{f \in F(m) \cup \{m\}} N_f^k(t) \leq   \sum_{f \in N_\gamma(m) \cup \{m\}} N_f^k(t)$, and since each message from a neighbor $f$ takes time $d(m, f) - 1$ time to reach agent $m$, we have that $\sum_{f \in F(m) \cup \{m\}} N_f^k(t) - \sum_{f \in N_\gamma(m) \cup \{m\}}(d(f, m)-1) \leq |S^m_k(t)|$. Using this in the previous equation and the fact that $d(m, f) \leq \gamma$, we have}
  \implies \sum_{f \in F(m) \cup \{m\}}N_f^k(t) &\leq cv^{\frac{1}{\varepsilon}}\ln t\left(\frac{2}{\Delta_k}\right)^{1+\frac{1}{\varepsilon}} + \gamma(|\mathcal N_\gamma(m)| + 1)
\end{align}
Therefore, we know that if $\beta_k^m \geq \ceil*{cv^{\frac{1}{\varepsilon}}\ln T\left(\frac{2}{\Delta_k}\right)^{1+\frac{1}{\varepsilon}} + \gamma(|\mathcal N_\gamma(m)| + 1)}$ then Event (C) does not occur. Additionally, using the union bound over $N^*_m(t)$ and $N^k_m(t)$, and Assumption~\ref{ass:robust_mean_estimator}, we have:
\begin{align}
    \mathbb P(\text{Event (A) or (B) occurs}) \leq 2\sum_{s=1}^t \frac{1}{s^4} \leq \frac{2}{t^3}.
\end{align}
Combining all probabilities, and inserting in the individual regret, we have,
\begin{align}
    R_\mathcal G(T) &\leq \sum_{k: \Delta_k > 0}\sum_{m \in V'} \Delta_k\left(\beta_k^m + (|F(m)|+1)\left(\sum_{t=1}^T \mathbb P\left\{A_{m, t} = k  ; U_k^m\right\}\right)\right) \\
    &\leq \sum_{k: \Delta_k > 0}\sum_{m \in V'} \Delta_k\left(\ceil*{ cv^{\frac{1}{\varepsilon}}\ln T\left(\frac{2}{\Delta_k}\right)^{1+\frac{1}{\varepsilon}} + \gamma(|\mathcal N_\gamma(m)| + 1)} + (|F(m)|+1)\left(\sum_{t=1}^T \frac{2}{t^3}\right)\right) \\
    &\leq \sum_{k: \Delta_k > 0}\sum_{m \in V'} \Delta_k\left( cv^{\frac{1}{\varepsilon}}\ln T\left(\frac{2}{\Delta_k}\right)^{1+\frac{1}{\varepsilon}} + \gamma(|\mathcal N_\gamma(m)| + 1) + (|F(m)|+2)\right) \\ \intertext{Since $|V'| \leq \alpha(\mathcal G_\gamma)$, where $\alpha(\cdot)$ denotes the independence number, we have,}
    &\leq \sum_{k: \Delta_k > 0}\left( 2^{1+1/\varepsilon}\Delta_k^{-1/\varepsilon}cv^{1/\varepsilon}\alpha(\mathcal G_\gamma)\ln T + \sum_{m \in V'} \Delta_k(\gamma(|\mathcal N_\gamma(m)| + 1) + (|F(m)|+2))\right)\\
    &\leq \left(\sum_{k:\Delta_k>0}\Delta_k^{-1/\varepsilon}\right)\left(2^{1+1/\varepsilon}cv^{1/\varepsilon}\alpha(\mathcal G_\gamma)\right)\ln T + \left(\alpha(\mathcal G_\gamma)M\gamma + M + \alpha(\mathcal G_\gamma)\right)\left(\sum_{k: \Delta_k > 0} \Delta_k\right)
\end{align}
\end{proof}
\begin{theorem}
For any agent $m \in \mathcal G$, let $m^*$ denote its corresponding leader agent ($m = m^*$ for leaders), and $F(m^*)$ denote the corresponding set of follower agents. Then, for any horizon $T > K$, individual regret for agent $m$ in \textsc{Centralized MP-UCB} is bounded by the following.
\begin{align*}
    R_m(T) \leq \sum_{k: \Delta_k > 0} \left(\frac{2^{1+\frac{1}{\varepsilon}}cv^{\frac{1}{\varepsilon}}}{\Delta^{1/\varepsilon}\left(|F(m^*)| + 1\right)}\right)\ln T + \left(\gamma + 3\right)\left(\sum_{k: \Delta_k > 0}\Delta_k + 1\right).
\end{align*}
For $T \leq K, R_m(T) \leq K$.
\end{theorem}
\begin{proof}
Let us consider any ``leader'' agent $m \in V'$. Let the subset of its neighbors in $\mathcal N_\gamma(m)$ that are assigned as its followers be $F(m)$. For any ``follower'' $f \in F(m)$ of the leader agent, we have the following.
\begin{align}
    a_f(t) &= a_m\left(t-d(m, f)\right) \\ \intertext{Therefore, for any arm $k$, we have,}
    \implies N_m^k\left(t-d(m,f)\right) &\leq N_f^k(t) \leq N^k_m\left(t-d(m,f)\right) + d(m,f) \\
    \implies N^k_m(t) - d(m,f) &\leq N^k_f(t) \leq N^k_m\left(t\right) + d(m,f).
\end{align}
Summing over agents $f \in F(m) \cup \{m\}$, we have,
\begin{equation}
    \left(|F(m)| + 1\right)N^k_m(t) - \sum_{f \in F(m)}d(m, f) \leq |S^m_k(t)| \leq \left(|F(m)| + 1\right) N^k_m(t) + \sum_{f \in F(m)}d(m, f)
    \label{eqn:individual_pulls}
\end{equation}
Now, let us consider the regret incurred by agent $m$.
\begin{align}
    R_m(T) &= \sum_{k: \Delta_k > 0}\Delta_k \mathbb E[N_m^k(T)] \\
    &= \sum_{k: \Delta_k > 0} \Delta_k \sum_{t=1}^T \mathbb P\left(A_{m, t} = k\right) \\ \intertext{For any constant $\eta_k > 0$, we have,}
    &\leq \sum_{k: \Delta_k > 0} \Delta_k \left(\eta_k + \sum_{t=1}^T \mathbb P\left(A_{m, t} = k, N_m^k(t) \geq \eta_k \right)\right).
\end{align}
We see that a suboptimal arm is pulled when one of three events occurs.
\begin{align}
    \text{Event (A): } &  \hat\mu_*^m(t-1) \leq \mu_* - v^{\frac{1}{1+\varepsilon}}\left(\frac{2c\ln t}{|S_m^*(t)|}\right)^{\frac{\varepsilon}{1+\varepsilon}} \\
    \text{Event (B): } &  \hat\mu_k^m(t-1) \geq \mu_k + v^{\frac{1}{1+\varepsilon}}\left(\frac{2c\ln t}{|S^m_k(t)|}\right)^{\frac{\varepsilon}{1+\varepsilon}} \\
    \text{Event (C): } &  \mu_* \leq \mu_k + 2v^{\frac{1}{1+\varepsilon}}\left(\frac{2c\ln t}{|S^m_k(t)|}\right)^{\frac{\varepsilon}{1+\varepsilon}} 
\end{align}
Now, let us examine the occurence of event $(C)$:
\begin{align}
  \Delta_k &\leq 2v^{\frac{1}{1+\varepsilon}}\left(\frac{2c\ln t}{|S^m_k(t)|}\right)^{\frac{\varepsilon}{1+\varepsilon}} \\
  \implies |S^m_k(t)| &\leq cv^{\frac{1}{\varepsilon}}\ln t\left(\frac{2}{\Delta_k}\right)^{1+\frac{1}{\varepsilon}} \\ \intertext{Replacing Equation~\ref{eqn:individual_pulls}, we get the following.}
  \implies N^k_m(t) &\leq \frac{cv^{\frac{1}{\varepsilon}}}{\left(|F(m)| + 1\right)}\ln t\left(\frac{2}{\Delta_k}\right)^{1+\frac{1}{\varepsilon}} + \frac{\sum_{f \in F(m)}d(m, f)}{\left(|F(m)| + 1\right)}
\end{align}
Hence, we know that whenever $N^k_m(t) > \frac{cv^{\frac{1}{\varepsilon}}}{\left(|F(m)| + 1\right)}\ln t\left(\frac{2}{\Delta_k}\right)^{1+\frac{1}{\varepsilon}} + \frac{\sum_{f \in F(m)}d(m, f)}{\left(|F(m)| + 1\right)}$, Event (C) will not occur. Hence, if we set $\eta_k = \ceil*{\frac{cv^{\frac{1}{\varepsilon}}}{\left(|F(m)| + 1\right)}\ln T\left(\frac{2}{\Delta_k}\right)^{1+\frac{1}{\varepsilon}} + \frac{\sum_{f \in F(m)}d(m, f)}{\left(|F(m)| + 1\right)}}$ only events (A) and (B) will occur. Additionally, using the union bound over $N^*_m(t)$ and $N^k_m(t)$, and Assumption~\ref{ass:robust_mean_estimator}, we have:
\begin{align}
    \mathbb P(\text{Event (A) or (B) occurs}) \leq 2\sum_{s=1}^t \frac{1}{s^4} \leq \frac{2}{t^3}.
\end{align}
Combining all probabilities, and inserting in the individual regret, we have,
\begin{align}
    R_m(T) &\leq \sum_{k: \Delta_k > 0} \Delta_k \left(\eta_k + \sum_{t=1}^T \mathbb P\left(A_{m, t} = k, N_m^k(t) \geq \eta_k \right)\right) \\
    &\leq \sum_{k: \Delta_k > 0} \Delta_k \left(\ceil*{\frac{cv^{\frac{1}{\varepsilon}}}{\left(|F(m)| + 1\right)}\ln T\left(\frac{2}{\Delta_k}\right)^{1+\frac{1}{\varepsilon}} + \frac{\sum_{f \in F(m)}d(m, f)}{\left(|F(m)| + 1\right)}} + \sum_{t=1}^T \frac{2}{t^3}\right) \\
    &\leq \sum_{k: \Delta_k > 0} \Delta_k \left(\frac{cv^{\frac{1}{\varepsilon}}}{\left(|F(m)| + 1\right)}\ln T\left(\frac{2}{\Delta_k}\right)^{1+\frac{1}{\varepsilon}} + \frac{\sum_{f \in F(m)}d(m, f)}{\left(|F(m)| + 1\right)} + 3\right) \\
    &\leq \sum_{k: \Delta_k > 0} \Delta_k \left(\frac{cv^{\frac{1}{\varepsilon}}}{\left(|F(m)| + 1\right)}\ln T\left(\frac{2}{\Delta_k}\right)^{1+\frac{1}{\varepsilon}} + \gamma + 3\right)\\
    &=\sum_{k: \Delta_k > 0} \left(\frac{2^{1+\frac{1}{\varepsilon}}cv^{\frac{1}{\varepsilon}}}{\Delta^{1/\varepsilon}\left(|F(m)| + 1\right)}\ln T\right) + \left(\gamma + 3\right)\sum_{k: \Delta_k > 0}\Delta_k
\end{align}
Let us now consider a follower agent $f \in F(m)$. We know that after at most $\gamma$ steps, the follower mimics agent $m$ exactly. Therefore, $R_f(T) \leq R_m(T) + \gamma$. Combining this with the earlier bound, we get the final form of the statement.
\end{proof}
\begin{theorem}
For any horizon $T > K$, and robust mean estimator that satisfies Assumption~\ref{ass:robust_mean_estimator} with constants $c$ and $v$, the group regret under \textsc{KMP-UCB} is bounded by the following.
\begin{align*}
R_\mathcal G(T) \leq \alpha(\mathcal G_\gamma)\left(\sum_{k: \Delta_k > 0} \Delta_k^{-1/\varepsilon}\right)\left(2cv^{\frac{1}{\varepsilon}}\right)\ln T + \alpha(\mathcal G_\gamma)\left(\gamma + 2 + M\right)\left(\sum_{k: \Delta_k > 0} \Delta_k\right).
\end{align*}
For $T \leq K, R_\mathcal G(T) \leq MK$. $\alpha(\cdot)$ denotes the independence number.
\end{theorem}
\begin{proof}
Consider the maximal independent set of $\mathcal G_\gamma$ given by $V'$. We can decompose the group regret based on $V'$ as follows.
\begin{align}
    R_\mathcal G(T) &= \sum_{m \in V'}R_m(T) + \sum_{m \in V \setminus V'}R_m(T) \\
    &\leq \sum_{m \in V'} \left(R_m(T) + \sum_{m' \in \mathcal N_\gamma(m)}R_{m'}(T)\right)
\end{align}
Now, for any agent $m$, consider the total regret for all agents in $\mathcal N_\gamma(m) \cup \{m\}$.
\begin{align}
    R_m(T) + \sum_{m' \in \mathcal N_\gamma(m)}R_{m'}(T) &= \sum_{k: \Delta_k > 0} \Delta_k\left(\sum_{m' \in \mathcal N_\gamma(m) \cup \{m\}} \sum_{t=1}^T \mathbb P\left(A_{m', t} = k\right)\right) \\ \intertext{For any set of constants $\eta^{m'}_k > 0, k \in [K], m' \in \mathcal N_\gamma(m) \cup \{m\}$, and $N^\dagger_k(t) = \sum_{m' \in \mathcal N_\gamma(m) \cup \{m\}} N^m_k(t), \beta_k = \sum_{m' \in \mathcal N_\gamma(m) \cup \{m\}} \eta_k^m$,}
    &\leq \sum_{k: \Delta_k > 0} \Delta_k\left(\beta_k + \sum_{m' \in \mathcal N_\gamma(m) \cup \{m\}}\sum_{t=1}^T \mathbb P\left(A_{m', t} = k, N^\dagger_k(t) > \beta_k\right)\right)
\end{align}
For any arm $k \in [K]$, at any trial $t$, any agent $m' \in \mathcal N_\gamma(m) \cup \{m\}$ chooses the confidence bound based on the pair $\hat\mu_k^{m^*}(t), N_k^{m^*}(t)$, where $N_k^{m^*}(t) = \max_{u \in \mathcal N_\gamma(m') \cup \{m'\}} |S_k^{u}(t-d(m, u))|$, and $\hat\mu_k^{m^*}(t)$ is the corresponding robust mean. Now, we know that $m'$ is in the $\gamma$-neighborhood of $m$, therefore,
\begin{align*}
    N_k^{m^*}(t) &= \max_{u \in \mathcal N_\gamma(m') \cup \{m'\}} |S_k^{u}(t-d(m, u))| \\
    &\geq |S_k^{m}(t-d(m, u))| \\
    &= N^\dagger_k(t-d(m, u)) \\
    &\geq N^\dagger_k(t) - d(m, u) \\
    &\geq N^\dagger_k(t) - \gamma.
\end{align*}
Now, any agent $m' \in \mathcal N_\gamma(m) \cup \{m\}$ pulls a suboptimal arm $k$ when one of three events occurs.
\begin{align}
    \text{Event (A): } &  \hat\mu_*^{m^*}(t-1) \leq \mu_* - v^{\frac{1}{1+\varepsilon}}\left(\frac{2c\ln t}{N_*^{m^*}(t)}\right)^{\frac{\varepsilon}{1+\varepsilon}} \\
    \text{Event (B): } &  \hat\mu_k^{m^*}(t-1) \geq \mu_k + v^{\frac{1}{1+\varepsilon}}\left(\frac{2c\ln t}{N_k^{m^*}(t)}\right)^{\frac{\varepsilon}{1+\varepsilon}} \\
    \text{Event (C): } &  \mu_* \leq \mu_k + 2v^{\frac{1}{1+\varepsilon}}\left(\frac{2c\ln t}{N_k^{m^*}(t)}\right)^{\frac{\varepsilon}{1+\varepsilon}} 
\end{align}
Now, let us examine the occurence of event $(C)$:
\begin{align}
  \Delta_k &\leq 2v^{\frac{1}{1+\varepsilon}}\left(\frac{2c\ln t}{N_k^{m^*}(t)}\right)^{\frac{\varepsilon}{1+\varepsilon}} \\
  \implies N_k^{m^*}(t) &\leq cv^{\frac{1}{\varepsilon}}\ln t\left(\frac{2}{\Delta_k}\right)^{1+\frac{1}{\varepsilon}}
\end{align}
Therefore, by setting $\beta_k = \ceil*{cv^{\frac{1}{\varepsilon}}\ln T\left(\frac{2}{\Delta_k}\right)^{1+\frac{1}{\varepsilon}} + \gamma}$, Event (C) will not occur for any agent $m' \in \mathcal N_\gamma(m) \cup \{m\}$. Additionally, using the union bound over $N^{m^*}_*(t)$ and $N^{m^*}_k(t)$, and Assumption~\ref{ass:robust_mean_estimator}, we have:
\begin{align}
    \mathbb P(\text{Event (A) or (B) occurs}) \leq 2\sum_{s=1}^t \frac{1}{s^4} \leq \frac{2}{t^3}.
\end{align}
Replacing this in the total regret for $m' \in \mathcal N_\gamma(m) \cup \{m\}$, we have,
\begin{align}
    R_m(T) + \sum_{m' \in \mathcal N_\gamma(m)}R_{m'}(T) &\leq \sum_{k: \Delta_k > 0} \Delta_k\left(\beta_k + \sum_{m' \in \mathcal N_\gamma(m) \cup \{m\}}\sum_{t=1}^T \mathbb P\left(A_{m', t} = k, N^\dagger_k(t) > \beta_k\right)\right) \\
    &\leq \sum_{k: \Delta_k > 0} \Delta_k\left(\ceil*{cv^{\frac{1}{\varepsilon}}\ln T\left(\frac{2}{\Delta_k}\right)^{1+\frac{1}{\varepsilon}} + \gamma} + \sum_{m' \in \mathcal N_\gamma(m) \cup \{m\}}\sum_{t=1}^T \frac{2}{t^3}\right) \\
    &\leq \sum_{k: \Delta_k > 0} \Delta_k\left(cv^{\frac{1}{\varepsilon}}\ln T\left(\frac{2}{\Delta_k}\right)^{1+\frac{1}{\varepsilon}} + \gamma + 2 + |\mathcal N_\gamma(m)| \right) \\ \intertext{Summing over all agents $m \in V'$, we have,}
    R_\mathcal G(T) &\leq \alpha(\mathcal G_\gamma)\left(\sum_{k: \Delta_k > 0} \Delta_k^{-1/\varepsilon}\right)\left(2cv^{\frac{1}{\varepsilon}}\right)\ln T + \alpha(\mathcal G_\gamma)\left(\gamma + 2 + M\right)\left(\sum_{k: \Delta_k > 0} \Delta_k\right).
\end{align}
\end{proof}
\newpage
\subsection{Full versions of the KMP-UCB algorithm}
\begin{algorithm}[H]
\caption{\textsc{KMP-UCB}}
\label{alg:mp_ucb_aug}
\small
\begin{algorithmic}[1] 
\STATE \textbf{Input}: $K$, $\varepsilon$, $\hat\mu_R(n, \delta)$, $c, v$.
\STATE Set $S^m_k = \phi\ \forall k \in [K]$, $Q_m(t) = \phi\ \forall m \in [M]$ 
\STATE Set $W_m = \bf 1^{|\mathcal N_\gamma(m)| \times 2K}\ \forall m \in [M]$.
\FOR{For each iteration $t \in [T]$}
\FOR{For each agent $m \in [M]$}
\IF{$t \leq K$}
\STATE $A_{m, t} = t$.
\ELSE
\FOR{Arm $k \in [K]$}
\STATE $\hat\mu^{(m)}_k(t) = \hat\mu_R(S^m_k, 2\ln t)$.
\ENDFOR
\STATE $W_m = \left(\hat\mu^{(m)}_k(t), |S^m_k|\right)_{k \in [K]}$.
\FOR{Arm $k \in [K]$}
\STATE $ m^* = \argmax_m \{s^m_k : (\mu^m_k, s^m_k) \in W\} $.
\STATE $\hat\mu^{m, *}_k, |S^{m, *}_k| = W_{m^*, k}$
\STATE $\text{UCB}_k^{(m)}(t) = v^{\frac{1}{1+\varepsilon}}\left(\frac{2c\ln t}{|S^{m,*}_k|}\right)^{\frac{\varepsilon}{1+\varepsilon}}$.
\ENDFOR
\STATE $A_{m, t} = \arg\max_{k \in [K]} \left\{\hat\mu_k ^{m, *}(t) + \text{UCB}_k^{(m)}(t)\right\}$.
\ENDIF
\STATE $X_{m, t} = \textsc{Pull}(A_{m, t})$.
\STATE $S_{A_{m,t}}^m = S_{A_{m,t}}^m \cup \{X_{m, t}\}$
\STATE $Q_m(t)$ = \textsc{PruneDeadMessages}($Q_m(t)$).
\STATE $q_m(t) = \left\langle m, t, \gamma, A_{m, t}, X_{m, t}, \hat{\bf{\mu}}^{(m)}(t), \bf N^{m}(t) \right\rangle$
\STATE $Q_m(t) = Q_m(t) \cup \{q_m(t)\}$.
\STATE Set $l =l -1 \ \forall \langle m',t',l, a', x', \bf d', \bf n'\rangle$ in $Q_m(t)$.
\FOR{Each neighbor $m'$ in $\mathcal N_1(m)$}
\STATE \textsc{SendMessages}$(m, m', Q_m(t))$.
\ENDFOR
\ENDFOR
\FOR{For each agent $m \in [M]$}
\STATE $Q_m(t+1) = \phi$.
\FOR{Each neighbor $m'$ in $\mathcal N_1(m)$}
\STATE $Q' = $\textsc{ReceiveMessages}$(m',m)$
\STATE $Q_m(t+1)$ = $Q_m(t+1) \cup Q'$.
\ENDFOR
\FOR{$\langle m',t',l',a',x', \bf d', \bf n'\rangle \ \in Q_m(t+1)$}
\STATE $S_{a'}^m = S_{a'}^m \cup \{x'\}$.
\STATE $W_{m'} = (\bf d', \bf n')$
\ENDFOR
\ENDFOR
\ENDFOR
\end{algorithmic}
\end{algorithm}
\newpage
\bibliographystyle{icml2019}
\bibliography{example_paper}
\end{document}